\documentclass{article}

\usepackage{microtype}
\usepackage{graphicx}
\usepackage{subfigure}
\usepackage{booktabs} 

\usepackage{hyperref}

\usepackage[accepted]{icml2018}

\usepackage[utf8]{inputenc} 
\usepackage[T1]{fontenc}    
\usepackage{amsthm}
\usepackage{url}            
\usepackage{booktabs}       
\usepackage{amsfonts}       
\usepackage{nicefrac}       
\usepackage{microtype}      
\usepackage{changepage}
\usepackage{multicol}
\usepackage{dsfont}
\usepackage{todonotes}
\usepackage{graphicx}
\usepackage{enumerate}
\usepackage{amsmath,amsfonts,amssymb}
\usepackage{bm}
\usepackage[font={small}]{caption}
\usepackage{mdframed}
\usepackage{wrapfig}
\usepackage{tabularx}
\newtheorem{definition}{Definition}
\newtheorem{corollary}{Corollary}
\newtheorem*{corollary*}{Corollary}
\newlength{\offsetpage}
{\end{adjustwidth}}

\newcommand\transp{^\intercal\kern-\scriptspace}

\usepackage{amsthm}
\usepackage{verbatim}

\newtheorem{theorem}{Theorem}

\newcommand{\boldy}{\mathbf{y}}
\newcommand{\boldx}{\mathbf{x}}

\newcommand{\boldz}{\mathbf{z}}
\newcommand{\boldc}{\mathbf{c}}
\newcommand{\boldh}{\mathbf{h}}
\newcommand{\boldW}{\mathbf{W}}

\newcommand{\boldb}{\mathbf{b}}
\newcommand{\bolds}{\mathbf{s}}

\newcommand{\reals}{\mathbb{R}}
\newcommand{\E}{\mathbb{E}}
\newcommand{\Prob}{\mathbb{P}}
\newcommand{\MLP}{\text{MLP}}

\newcommand{\softmax}{\text{softmax}}

\newcommand{\tboldz}{\tilde{\mathbf{z}}}

\newcommand{\irchi}[2]{\raisebox{\depth}{$#1\chi$}}
\newcommand{\squi}{\rightsquigarrow}

\DeclareMathOperator*{\argmax}{arg\,max}
\DeclareRobustCommand{\rchi}{{\mathpalette\irchi\relax}}

{\vspace{-\topsep}\begin{itemize}\itemsep1pt \parskip0pt \parsep0pt \topsep1pt}
{\end{itemize}\vspace{-\topsep}}

{\vspace{-\topsep}\begin{itemize}[leftmargin=*]\itemsep1pt \parskip0pt \parsep0pt \topsep1pt}
{\end{itemize}\vspace{-\topsep}}
\usepackage{enumitem}

\newtheorem{prop}{Proposition}
\usepackage{changepage} 

\icmltitlerunning{Adversarially Regularized Autoencoders}

\begin{document}

\twocolumn[
\icmltitle{Adversarially Regularized Autoencoders}



\icmlsetsymbol{equal}{*}

\begin{icmlauthorlist}
\icmlauthor{Jake (Junbo) Zhao}{equal,nyu,fb}
\icmlauthor{Yoon Kim}{equal,harvard}
\icmlauthor{Kelly Zhang}{nyu}
\icmlauthor{Alexander M. Rush}{harvard}
\icmlauthor{Yann LeCun}{nyu,fb}
\end{icmlauthorlist}

\icmlaffiliation{nyu}{Department of Computer Science, New York University}
\icmlaffiliation{harvard}{School of Engineering and Applied Sciences, Harvard University}
\icmlaffiliation{fb}{Facebook AI Research}

\icmlcorrespondingauthor{Jake Zhao}{jakezhao@cs.nyu.edu}

\icmlkeywords{Machine Learning, ICML}

\vskip 0.3in
]



\printAffiliationsAndNotice{\icmlEqualContribution} 

\begin{abstract}
Deep latent variable models, trained using variational autoencoders or
generative adversarial networks, are now a key technique for representation learning of continuous structures.
However, applying similar methods to discrete structures, such as text sequences or
discretized images, has proven to be more challenging.
In this work, we propose a flexible method for training deep latent
variable models of discrete structures. Our approach is based on the
recently-proposed Wasserstein autoencoder (WAE) which formalizes the
adversarial autoencoder (AAE) as an optimal transport problem.
We first extend this framework to model discrete sequences, and
then further explore different learned priors targeting a controllable representation.
This adversarially regularized autoencoder (ARAE)
allows us to generate natural textual outputs as well as perform manipulations in the latent space to induce change in the output space.
Finally we show that the latent representation can be trained to perform unaligned textual style transfer,
giving improvements both in automatic/human evaluation compared to existing methods.

\end{abstract}

\vspace{-7mm}
\section{Introduction}
\vspace{-2mm}

Recent work on deep latent variable models, such as variational autoencoders \citep{Kingma2014}
and generative adversarial networks \cite{goodfellow2014generative}, has shown significant progress in learning smooth representations of complex, high-dimensional continuous data such as
images. These latent variable representations facilitate the ability to apply smooth transformations in latent space in order to produce complex modifications of generated outputs, while still remaining on the data
manifold.

Unfortunately, learning similar latent variable models of discrete
structures, such as text sequences or discretized images, remains a
challenging problem. Initial work on VAEs for text has shown
that optimization is difficult, as the generative model can easily degenerate
into a unconditional language model \citep{bowman2015generating}. Recent
work on generative adversarial networks (GANs) for text has mostly focused on dealing with the non-differentiable
objective either through policy gradient methods \citep{Che2017, Hjelm2018, Yu2017} or with the
Gumbel-Softmax distribution \citep{Kusner2017}. However, neither approach can
yet produce robust representations directly.

In this work, we extend the adversarial autoencoder (AAE) \cite{Makhzani2015} to discrete sequences/structures. Similar to the AAE, our model learns an encoder from
an input space to an adversarially regularized continuous latent space. However unlike the AAE which utilizes a fixed prior, we instead learn a parameterized prior as a GAN.
Like sequence VAEs, the model does not require
using policy gradients or continuous relaxations. Like GANs, the model provides
flexibility in learning a prior through a parameterized generator.

This adversarially regularized autoencoder (ARAE) can further be formalized under the recently-introduced Wasserstein autoencoder (WAE)
framework~\citep{tolstikhin2017wasserstein}, which also generalizes the adversarial autoencoder.
This framework connects  regularized autoencoders to an optimal transport objective for an
implicit generative model. We extend this class of latent variable models to the case of discrete output,
specifically showing that the autoencoder cross-entropy loss upper-bounds the total variational distance between the model/data distributions.
Under this setup, commonly-used discrete decoders such as RNNs, can be incorporated into the model.
Finally to handle non-trivial sequence examples, we consider several different
(fixed and learned) prior distributions. These include a standard Gaussian prior
used in image models and in the AAE/WAE models, a learned parametric
generator acting as a GAN in latent variable space, and a transfer-based
parametric generator that is trained to ignore targeted attributes of the input.
The last prior can be directly used for unaligned transfer tasks such as
sentiment or style transfer.

Experiments apply ARAE to discretized images and text sequences.
The latent variable model is able to generate varied samples that
can be quantitatively shown to cover the input spaces and to generate consistent image and sentence manipulations by moving around in the latent space via interpolation and offset vector arithmetic.
When the ARAE model is trained with task-specific adversarial regularization, the model improves upon strong results on sentiment transfer reported in \citet{Shen2017} and produces compelling outputs on a topic transfer task using only a single shared space. Code is available at \url{https://github.com/jakezhaojb/ARAE}.

\vspace{-3mm}
\section{Background and Notation}
\vspace{-2mm}
\paragraph{Discrete Autoencoder}
Define $\mathcal{X} = \mathcal{V}^n$ to be a set
of discrete sequences where $\mathcal{V}$ is a vocabulary of symbols. Our discrete autoencoder will consist of two parameterized
functions: a deterministic encoder function
$\text{enc}_{\phi}: \mathcal{X} \mapsto \mathcal{Z}$ with parameters
$\phi$ that maps from input space to code space, and a conditional decoder  $p_{\psi}(\boldx\ |\ \boldz)$  over structures
$\mathcal{X}$ with parameters $\psi$. The parameters are trained based on the cross-entropy reconstruction
loss:
\begin{equation*}
\mathcal{L}_{\text{rec}}(\phi, \psi) = - \log p_{\psi}(\boldx\ |\ \text{enc}_{\phi}(\boldx))
\end{equation*}
The choice of the encoder and decoder parameterization is problem-specific,
for example we use RNNs for sequences. We use the notation,
$\hat{\boldx} = \arg\max_{\boldx} p_{\psi}(\boldx\ |\
\text{enc}_{\phi}(\boldx)) $
for the decoder mode, and call the model distribution $\Prob_{\psi}$.

\vspace{-1mm}
\paragraph{Generative Adversarial Networks}
GANs are a class of parameterized implicit generative models \citep{goodfellow2014generative}.  The
method approximates drawing samples from a true distribution
$\boldz \sim \Prob_{*}$ by instead employing a noise sample
$\bolds$ and a parameterized generator function
$\tboldz = g_{\theta}(\bolds)$ to produce
$\tboldz \sim \Prob_{\boldz}$. Initial work on GANs implicitly 
minimized the Jensen-Shannon divergence between
the distributions. Recent work on Wasserstein GAN (WGAN)
\citep{arjovsky2017wasserstein}, replaces this with
the \emph{Earth-Mover} (Wasserstein-1) distance.

GAN training utilizes two separate models: a \textit{generator}
$g_\theta(\bolds)$ maps a latent vector from some easy-to-sample noise
distribution to a sample from a more complex distribution, and a critic/discriminator $f_w(\boldz)$
aims to distinguish \emph{real} data and \emph{generated} samples from $g_\theta$. Informally, the generator is trained to fool the
critic, and the critic to tell real from generated.
WGAN training uses the following min-max optimization over generator $\theta$ and critic $w$,
\begin{equation*}
\label{equ:wgan}
\min_{\theta} \max_{w \in \mathcal{W}} \E_{\boldz \sim \Prob_{*}}[f_w(\boldz)] - \E_{\tboldz \sim \Prob_{\textbf{z}}}[f_w(\tboldz)],
\end{equation*}
where $f_w:\mathcal{Z} \mapsto \reals$ denotes the critic function,
$\tboldz$ is obtained from the generator,
$\tboldz=g_{\theta}(\bolds)$, and $\Prob_{*}$ and $\Prob_{\boldz}$ are real
and generated distributions. If the critic parameters $w$ are
restricted to an 1-Lipschitz function set $\mathcal{W}$, this term
correspond to minimizing Wasserstein-1 distance
$W(\Prob_{*}, \Prob_{\boldz})$. We use a naive approximation to enforce
this property by weight-clipping, i.e.  $w = [-\epsilon, \epsilon]^d$
\citep{arjovsky2017wasserstein}.\footnote{While we did not experiment with enforcing the Lipschitz constraint via gradient penalty \cite{Gulrajani2017} or
spectral normalization \cite{Miyato2018}, other researchers have found slight improvements by training ARAE with the gradient-penalty version of WGAN (private correspondence).}

\vspace{-3mm}
\section{Adversarially Regularized Autoencoder}
\label{sec:model}
\vspace{-2mm}

ARAE combines a discrete autoencoder with a GAN-regularized latent
representation. The full model is shown in Figure~\ref{fig:diagram}, which produces 
a learned distribution over the discrete space $\Prob_{\psi}$.
Intuitively, this method aims to provide smoother
hidden encoding for discrete sequences with a flexible prior.  In the
next section we show how this simple network can be formally
interpreted as a latent variable model under the Wasserstein
autoencoder framework.

The model consists of a discrete autoencoder
regularized with a prior distribution,
\[   \min_{\phi, \psi}\quad    \mathcal{L}_{\text{rec}}(\phi, \psi) + \lambda^{(1)} W(\Prob_Q, \Prob_{\mathbf{z}}) \]
\noindent
Here $W$ is the Wasserstein distance between $\Prob_Q$, the
distribution from a discrete encoder model
(i.e. $\text{enc}_\phi(\boldx)$ where $\boldx \sim \Prob_\star$), and $\Prob_{\mathbf{z}}$, a prior distribution. As above, the $W$ function is computed with
an embedded critic function which is optimized adversarially to the
generator and encoder.\footnote{Other GANs could be
used for this optimization. Experimentally we
found that WGANs to be more stable than other models.}

The model is trained with coordinate descent across: (1) the encoder
and decoder to minimize reconstruction, (2) the critic function to
approximate the $W$ term, (3) the encoder adversarially to the critic
to minimize $W$:
{\footnotesize
\vspace{-1mm}
\begin{align*}
  1) & \min_{\phi, \psi}  &  \mathcal{L}_{\text{rec}}(\phi, \psi) &= \E_{\boldx \sim \Prob_{\star}}\left[- \log p_\psi(\boldx \,|\, \text{enc}_\phi(\boldx)) \right]\\
  2) & \max_{w \in \mathcal{W}} &  \mathcal{L}_{\text{cri}}(w) &=   \E_{\boldx \sim  \Prob_\star}\left[ f_w(\text{enc}_{\phi}(\boldx))\right] - \E_{\tboldz \sim \Prob_{\mathbf{z}}}\left[f_w(\tboldz)\right] \\
  3) & \min_{\phi} & \mathcal{L}_{\text{enc}}(\phi) &=  \E_{\boldx \sim  \Prob_\star}\left[  f_w(\text{enc}_{\phi}(\boldx))\right]  -  \E_{\tboldz \sim \Prob_{\mathbf{z}}}\left[f_w(\tboldz)\right]
\end{align*}
\vspace{-3mm}
}
The full training algorithm is shown in Algorithm~1.

Empirically we found that the choice of the prior distribution
$\Prob_{\mathbf{z}}$ strongly impacted the performance of the model.
The simplest choice is to use a fixed distribution such as a Gaussian
${\cal N}(0, I)$, which yields a discrete version of the
adversarial autoencoder (AAE). However in practice this choice is
seemingly too constrained and suffers from mode-collapse.\footnote{We note that recent work has successfully utilized AAE for text by instead employing a spherical prior \cite{Cifka2018}.}

Instead we exploit the adversarial setup and use learned prior
parameterized through a generator model. This is analogous to the use of learned priors in VAEs \cite{Chen2017,Tomczak2017}.
Specifically we introduce a
generator model, $g_\theta(\bolds)$ over noise
$\bolds \sim {\cal N}(0, I)$ to act as an implicit prior distribution $\Prob_{\mathbf{z}}$.\footnote{The downside of this approach is that the latent variable
$\mathbf{z}$ is now much less constrained. However we find experimentally that
using a a simple MLP for $g_{\theta}$ significantly regularizes the encoder RNN.} We optimize its parameters $\theta$ as part of
training in Step 3.

\begin{figure}[t]
  \centering
  \begin{mdframed}
  \includegraphics[width=1\linewidth]{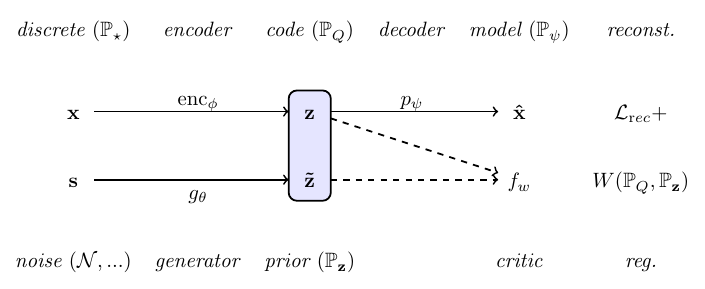}
  \end{mdframed}
  \vspace{-3mm}
  \caption{\label{fig:diagram} ARAE architecture.  A discrete sequence
    $\mathbf{x}$ is encoded and decoded to produce $\mathbf{\hat{x}}$.
    A noise sample $\mathbf{s}$ is passed though a generator
    $g_\theta$ (possibly the identity) to produce a prior. The critic function $f_w$
    is only used at training to enforce regularization $W$.  The model
    produce discrete samples
    $\mathbf{x}$ from noise $\mathbf{s}$. Section 5 relates these samples
    $\mathbf{x} \sim \Prob_{\psi}$ to $\mathbf{x} \sim \Prob_{\star}$.
}
  \vspace{-3mm}
\end{figure}
{\footnotesize
\begin{algorithm}[t]
	\footnotesize
  \caption{ARAE Training}\label{alg:train}
  \begin{algorithmic}
  \FOR{each training iteration}
      \STATE \textbf{\textit{(1) Train the encoder/decoder for reconstruction}} $(\phi, \psi)$ 
      \STATE  Sample $\{\boldx^{(i)}\}_{i=1}^m \sim \Prob_{\star}$ and compute $\boldz^{(i)} = \text{enc}_{\phi}(\boldx^{(i)})$
      \STATE  Backprop  loss, $\mathcal{L}_{\text{rec}} = -\frac{1}{m} \sum_{i=1}^m{\log p_{\psi}(\boldx^{(i)}\ |\ \boldz^{(i)})}$
      \vspace{0.2cm}
      \STATE \textbf{\textit{(2) Train the critic}} $(w)$ 
      \STATE  Sample $\{\boldx^{(i)}\}_{i=1}^m \sim \Prob_{\star}$ and $\{\bolds^{(i)}\}_{i=1}^m \sim \mathcal{N}(0, \mathbf{I})$
      \STATE  Compute $\boldz^{(i)} = \text{enc}_{\phi}(\boldx^{(i)})$ and $\tboldz^{(i)} = g_{\theta}(\boldz^{(i)})$
      \STATE  Backprop loss $-\frac{1}{m} \sum_{i=1}^m f_w(\boldz^{(i)}) +\frac{1}{m} \sum_{i=1}^m f_w(\tboldz^{(i)})$
      \STATE Clip  critic $w$ to $[-\epsilon, \epsilon]^{d}$.
      \vspace{0.2cm}
      \STATE \textbf{\textit{(3) Train the encoder/generator adversarially}} $(\phi, \theta)$ 
      \STATE  Sample $\{\boldx^{(i)}\}_{i=1}^m \sim \Prob_{\star}$ and $\{\bolds^{(i)}\}_{i=1}^m \sim \mathcal{N}(0, \mathbf{I})$
      \STATE  Compute $\boldz^{(i)} = \text{enc}_{\phi}(\boldx^{(i)})$ and $\tboldz^{(i)} = g_{\theta}(\bolds^{(i)})$.
      \STATE  Backprop loss  $ \frac{1}{m} \sum_{i=1}^m f_w(\boldz^{(i)}) - \frac{1}{m} \sum_{i=1}^m f_w(\tboldz^{(i)})$
  \ENDFOR
  \end{algorithmic}
\end{algorithm}
\begin{algorithm}
   \caption{ARAE Transfer Extension}\label{alg:train2}
  \begin{algorithmic}
    \STATE [Each loop additionally:]
    \STATE \textbf{\textit{(2b) Train attribute classifier}} $(u)$ 
    \STATE  Sample $\{\boldx^{(i)}\}_{i=1}^m \sim \Prob_{\star}$, lookup $y^{(i)}$, and compute $\boldz^{(i)} = \text{enc}_{\phi}(\boldx^{(i)})$
    \STATE  Backprop loss $-\frac{1}{m} \sum_{i=1}^m \log p_u(y^{(i)} | \boldz^{(i)})$
    \vspace{0.2cm}
    \STATE \textbf{\textit{(3b) Train the encoder adversarially}} $(\phi)$ 
    \STATE  Sample $\{\boldx^{(i)}\}_{i=1}^m \sim \Prob_{\star}$, lookup $y^{(i)}$, and compute $\boldz^{(i)} = \text{enc}_{\phi}(\boldx^{(i)})$
    \STATE   Backprop loss  $-\frac{1}{m} \sum_{i=1}^m \log p_u(1-y^{(i)}\ |\ \boldz^{(i)})$
  \end{algorithmic}
\end{algorithm}
}
\vspace{-3mm}
\paragraph{Extension: Unaligned Transfer}
\label{sub:codetrans}
Regularization of the latent space makes it more adaptable for direct
continuous optimization that would be difficult over discrete
sequences.  For example, consider the problem of unaligned transfer,
where we want to change an attribute of a discrete input without
aligned examples, e.g. to change the topic or sentiment of a
sentence. Define this attribute as $y$ and redefine the decoder
to be conditional $p_{\psi}(\boldx\ |\ \boldz, y)$.

To adapt ARAE to this setup, we modify the objective to learn to remove
attribute distinctions from the prior (i.e. we want the prior to encode
all the relevant information \emph{except} about $y$).
Following similar techniques
from other domains, notably in images~\citep{Lample2017} and video
modeling~\citep{denton2017unsupervised}, we introduce a latent space attribute
classifier:
\[   \min_{\phi, \psi, \theta}\quad    \mathcal{L}_{\text{rec}}(\phi, \psi) + \lambda^{(1)} W(\Prob_{Q}, \Prob_{\mathbf{z}})  - \lambda^{(2)} {\cal L}_{\text{class}}(\phi, u)\]
where ${\cal L}_{\text{class}}(\phi, u)$ is the loss of a classifier $p_{u}(y\ |\ \boldz)$ from latent variable to labels (in our experiments we always set $\lambda^{(2)} =1$).
This requires two more update steps: (2b) training the classifier, and (3b) adversarially training the encoder to this classifier. This algorithm is shown in Algorithm~\ref{alg:train2}.

\vspace{-3mm}
\section{Theoretical Properties}
\vspace{-1mm}
\label{sec:was}
Standard GANs implicitly minimize a divergence measure (e.g. $f$-divergence or Wasserstein distance) between the true/model distributions. In our case however, we implicitly minimize the divergence between \emph{learned} code distributions, and it is not clear if this training objective is matching the distributions in the original discrete space. \citet{tolstikhin2017wasserstein} recently showed that this style of training is minimizing the Wasserstein distance between the data distribution $\Prob_\star$ and the model distribution $\Prob_\psi$ with latent variables (with density $ p_\psi(\boldx) = \int_{\boldz} p_\psi(\boldx\ |\ \boldz)\ p(\boldz)\ d\boldz$).

In this section we apply the above result to the discrete case and show that the ARAE loss minimizes an upper bound on the \emph{total variation distance}
between $\Prob_\star$ and $\Prob_\psi$.
\begin{definition}[Kantorovich's formulation of optimal transport]{Let $\Prob_\star, \Prob_\psi$ be distributions over $\mathcal{X}$, and further let $c(\boldx,\boldy): \mathcal{X} \times \mathcal{X} \rightarrow \mathbb{R}^{+}$ be a cost function. Then the \emph{optimal transport (OT)} problem is given by
\vspace{-2mm}
\[W_c(\Prob_\star, \Prob_\psi) = \inf_{\Gamma \in \mathcal{P}(\boldx \sim \Prob_\star, \boldy \sim \Prob_\psi)} \mathbb{E}_{\boldx,\boldy \sim \Gamma}[c(\boldx, \boldy)] \]
where $\mathcal{P}(\boldx \sim \Prob_\star, \boldy \sim \Prob_\psi)$ is the set of all joint distributions of $(\boldx,\boldy)$ with marginals $\Prob_\star$ and $\Prob_\psi$.}
\end{definition}
In particular, if $c(\boldx,\boldy) = \Vert \boldx-\boldy \Vert_p^p $ then $W_c(\Prob_\star, \Prob_\psi)^{\frac{1}{p}}$ is the Wasserstein-$p$ distance between $\Prob_\star$ and $\Prob_\psi$.
Now suppose we utilize a latent variable model to fit the data, i.e. $\boldz \sim \Prob_\boldz, \boldx \sim \Prob_\psi(\boldx\ |\ \boldz)$. Then \citet{tolstikhin2017wasserstein} prove the following theorem:
\begin{theorem}{
Let $G_\psi: \mathcal{Z} \rightarrow \mathcal{X}$ be a deterministic function (parameterized by $\psi$) from the latent space $\mathcal{Z}$ to data space $\mathcal{X}$ that induces a dirac distribution $\Prob_\psi(\boldx\ |\ \boldz)$ on $\mathcal{X}$, i.e. $p_\psi(\boldx\ |\ \boldz) = \mathds{1} \{\boldx = G_\psi(\boldz)\}$. Let $Q(\boldz\ |\ \boldx)$ be any conditional distribution on $\mathcal{Z}$ with density $p_Q(\boldz\ |\ \boldx)$. Define its marginal to be $\Prob_Q$, which has density $p_Q(\boldx) = \int_\boldx p_Q(\boldz\ |\ \boldx)\ p_\star(\boldx)d\boldx$. Then,
\[ W_c(\Prob_\star,\Prob_\psi) = \inf_{Q(\boldz\ |\ \boldx) : \Prob_Q = \Prob_\boldz} \mathbb{E}_{\Prob_\star}\mathbb{E}_{Q(\boldz\ |\ \boldx)} [c(\boldx, G_\psi(\boldz))] \] }
\end{theorem}
Theorem 1 essentially says that learning an autoencoder can be interpreted as learning a generative model with latent variables, as long as we ensure that the marginalized encoded space is the same as the prior. This provides theoretical justification for adversarial autoencoders \cite{Makhzani2015}, and \citet{tolstikhin2017wasserstein} used the above to train deep generative models of images by minimizing the Wasserstein-2 distance (i.e. squared loss between real/generated images). We now apply Theorem 1 to discrete autoencoders trained with cross-entropy loss.
\begin{corollary}[Discrete case] Suppose $\boldx \in \mathcal{X}$ where $\mathcal{X}$ is the set of all one-hot vectors of length $n$, and let $f_\psi:\mathcal{Z} \rightarrow \Delta^{n-1}$ be a deterministic function that goes from the latent space $\mathcal{Z}$ to the $n-1$ dimensional simplex $\Delta^{n-1}$.
Further let $G_\psi: \mathcal{Z} \rightarrow \mathcal{X}$ be a deterministic function such that $G_\psi(\boldz)= \argmax_{\mathbf{w} \in \mathcal{X}}\mathbf{w}^\top f_\psi(\boldz)$, and as above let $\Prob_\psi(\boldx\ |\ \boldz)$ be the dirac distribution derived from $G_\psi$ such that $p_\psi(\boldx\ |\ \boldz) = \mathds{1}\{\boldx = G_\psi(\boldz) \}$.  Then the following is an upper bound on $\Vert \Prob_\psi - \Prob_\star \Vert_{\textup{TV}}$, the total variation distance between $\Prob_\star$ and $\Prob_\psi$:
\vspace{-2mm}
\[  \inf_{Q(\boldz\ |\ \boldx) : \Prob_Q = \Prob_\boldz} \mathbb{E}_{\Prob_\star}\mathbb{E}_{Q(\boldz\ |\ \boldx)} \Big[-\frac{2}{\log 2} \log \boldx^\top f_\psi(\boldz)\Big] \]
\end{corollary}
The proof is in Appendix~\ref{sec:proof}. For natural language we have $n = |\mathcal{V}|^m$ and therefore $\mathcal{X}$ is the set of sentences of length $m$, where $m$ is the maximum sentence length (shorter sentences are padded if necessary). Then the total variational (TV) distance is given by
\vspace{-2mm}
\[\Vert \Prob_\psi - \Prob_\star \Vert_{\text{TV}} = \frac{1}{2} \sum_{\boldx \in \mathcal{V}^m} |p_\psi(\boldx) - p_{\star}(\boldx)|\]
This is an interesting alternative to the usual maximum likelihood approach which instead minimizes $\text{KL}(\Prob_\star, \Prob_\psi)$.\footnote{The relationship between KL-divergence and total variation distance is also given by Pinsker's inquality, which states that $2\Vert \Prob_{\psi} - \Prob_{\star} \Vert_{\text{TV}}^2 \le \text{KL}(\Prob_\star, \Prob_\psi)$.} It is also clear that $-\log \boldx^\top f_\psi(\boldz) = -\log p_\psi(\boldx\ |\ \boldz)$, the standard autoencoder cross-entropy loss at the sentence level with $f_\psi$ as the decoder. As the above objective is hard to minimize directly, we follow \citet{tolstikhin2017wasserstein} and consider an easier objective by (i) restricting $Q(\boldz\ |\ \boldx)$ to a family of distributions induced by a deterministic encoder parameterized by $\phi$, and (ii) using a Langrangian relaxation of the constraint $\Prob_Q = \Prob_\boldz$. In particular, letting $Q(\boldz\ |\ \boldx) = \mathds{1}\{ \boldz = \text{enc}_\phi(\boldx) \}$ be the dirac distribution induced by a deterministic encoder (with associated marginal $\Prob_\phi$), the objective is given by
\vspace{-2mm}
\[ \min_{\phi, \psi} \E_{\Prob_\star}[-\log p_\psi(\boldx\ |\ \text{enc}_\phi(\boldz))]  + \lambda W(\Prob_\phi, \Prob_\boldz)\]
Note that our minimizing the Wasserstein distance in the \emph{latent} space $W(\Prob_\phi, \Prob_\boldz)$ is independent from the Wassertein distance minimization in the \emph{output} space in WAEs.
Finally, instead of using a fixed prior (which led to mode-collapse in our experiments) we parameterize $\Prob_\boldz$ implicitly by transforming a simple random variable with a generator (i.e. $\bolds \sim \mathcal{N}(0, I), \boldz = g_\theta(\bolds))$. This recovers the ARAE objective from the previous section.

We conclude this section by noting that while the theoretical formalization of the AAE as a latent variable model was an important step, in practice there are many approximations made to the actual optimal transport objective. Meaningfully quantifying (and reducing) such approximation gaps remains an avenue for future work. 
\vspace{-4mm}
\section{Methods and Architectures}
\vspace{-1mm}

We experiment with ARAE on three setups: (1) a small
model using discretized images trained on the binarized version of
MNIST, (2) a model for text sequences trained  on the Stanford
Natural Language Inference (SNLI) corpus \citep{Bowman2015}, and (3) a
model trained for text transfer trained on the Yelp/Yahoo datasets
for unaligned sentiment/topic transfer.
For experiments using a learned prior, the generator
architecture uses a low dimensional $\bolds$ with a Gaussian prior
$\bolds \sim \mathcal{N}(0, \mathbf{I})$, and maps it to
$\boldz$ using an MLP $g_\theta$. The critic $f_w$ is also
parameterized as an MLP.

The \textit{image} model encodes/decodes
binarized images. Here $\mathcal{X} = \{0,1\}^{n}$ where $n$ is the
image size.  The encoder used is an MLP mapping
from $\{0,1\}^n \mapsto \reals^m$,
$\text{enc}_\phi(\boldx) = \MLP(\boldx; \phi) = \boldz$. The decoder
predicts each pixel in $\boldx$ with as a parameterized logistic
regression,
$p_\psi(\boldx\ |\ \boldz) = \prod_{j=1}^n \sigma(\boldh)^{x_j} (1-
\sigma(\boldh))^{1-x_j}$ where $\boldh = \MLP(\boldz; \psi)$.

The \textit{text} model uses a recurrent neural network (RNN) for both the encoder and
decoder. Here  $\mathcal{X} = \mathcal{V}^{n}$ where $n$ is the sentence length
and $\mathcal{V}$ is the vocabulary of the underlying language.
We define $\text{enc}_{\phi}(\boldx) = \boldz$ to be the last hidden state of an encoder RNN.
For decoding we feed $\boldz$ as an additional input to the decoder RNN at each time step, and
calculate the distribution over $\mathcal{V}$ at each time step via softmax,
$p_{\psi}(\boldx \ | \ \boldz) = \prod_{j=1}^n \softmax(\boldW
h_j + \boldb)_{x_j} $
where $\boldW$ and $\boldb$ are parameters (part of $\psi$) and $h_j$ is the decoder RNN hidden state.
To be consistent with Corollary 1 we need to find the highest-scoring sequence $\hat{\boldx}$ under this distribution during decoding, which is intractable in general. Instead we approximate this with greedy search.
The \textit{text transfer} model uses the same architecture as the text model but
extends it with a classifier $p_u(y\ |\ \boldz)$ which is modeled using
an MLP and trained to minimize cross-entropy.

We further compare our approach with a standard autoencoder (AE) and the
cross-aligned autoencoder \citep{Shen2017} for transfer.  In
both our ARAE and standard AE experiments, the  encoder output is
normalized to lie on the unit sphere, and the generator output is
bounded to lie in $(-1,1)^n$ by the $\tanh$ function at output layer.


Note, learning deep latent variable models for text sequences has been a
significantly more challenging empirical problem than for images.
Standard models such as VAEs suffer from optimization
issues that have been widely documented.  We performed 
experiments with recurrent VAE, introduced by
\cite{bowman2015generating}, as well as the adversarial autoencoder
(AAE) \cite{Makhzani2015}, both with Gaussian priors. We found that neither model was able to learn meaningful latent
representations---the VAE simply ignored the latent code and the AAE
experienced mode-collapse and repeatedly generated the same samples.\footnote{However there have been some recent successes training such models, as noted in the related works section}
Appendix \ref{app:exp} includes detailed descriptions of the
hyperparameters, model architecture, and training regimes.

\vspace{-3mm}
\section{Experiments}
\vspace{-1mm}

\subsection{Distributional Coverage}
\vspace{-2mm}

\begin{figure}
  \centering
\includegraphics[width=0.45\linewidth]{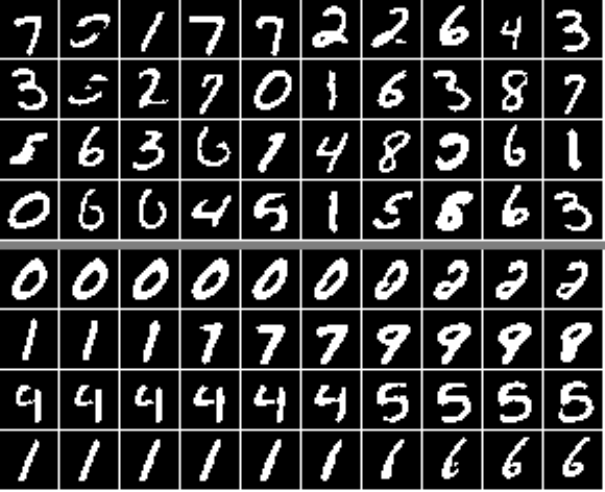}
  \vspace{-2mm}
  \caption{ Image samples. The top block shows output generation of the decoder for random noise samples; the bottom block shows sample interpolation results.}
      \vspace{-4mm}
  \label{fig:mnist}

\end{figure}

\begin{table}
\small
  \centering

  \begin{tabular}{lrr}
    \toprule
    Data  & Reverse PPL & Forward PPL\\
    \midrule
    Real data &  27.4 & - \\
    LM samples  & 90.6  & 18.8\\
    AE samples  & 97.3 & 87.8 \\
    ARAE samples & 82.2 & 44.3 \\
    \bottomrule
  \end{tabular}
  \vspace{-1mm}
  \caption{\label{tab:revppl}  Reverse PPL: Perplexity
    of language models trained on the synthetic
    samples from a ARAE/AE/LM, and evaluated on real data. Forward PPL: Perplexity of a language model trained on real data and evaluated on synthetic samples.}
        \vspace{-5mm}
\end{table}

Section~\ref{sec:was} argues that $\Prob_\psi$
is trained to approximate the true data distribution over discrete sequences $\Prob_\star$.
While it is difficult to test for this property directly (as is the case with
most GAN models), we can take samples from model to test the fidelity and coverage
of the data space. Figure~\ref{fig:mnist} shows a set of samples from
discretized MNIST and Appendix~\ref{app:generation} shows a set of generations from
the text ARAE.

A common quantitative measure of sample quality for generative models is to evaluate a
strong surrogate model trained on its generated samples.  While there
are pitfalls of this style of evaluation methods \citep{Theis2016}, it has provided a
starting point for image generation models. Here we use a similar
method for text generation, which we call \emph{reverse perplexity}. We generate 100k samples from each of the models, train an
  RNN language model on generated samples and evaluate perplexity on held-out data.\footnote{We also found this metric to be helpful for early-stopping.}   While similar metrics for images
  (e.g. Parzen windows)  have been shown to be problematic, we  argue that this is less of an issue
  for text as RNN language models achieve state-of-the-art perplexities on text datasets.
 We also calculate the usual ``forward'' perplexity by training an RNN language model on real data and testing on generated data. This measures the fluency of the generated samples, but cannot detect mode-collapse, a common issue in training GANs \cite{Arjovsky2017,Hu2018}.

Table~\ref{tab:revppl} shows these metrics for (i) ARAE, (ii) an
autoencoder (AE),\footnote{To ``sample'' from an AE we fit a multivariate
  Gaussian to the code space after training and generate code vectors
  from this Gaussian to decode back into sentence space.} (iii) an RNN
language model (LM), and (iv) the real training set. We
further find that with a fixed prior, the reverse perplexity of an AAE-style text model
\citep{Makhzani2015} was quite high (980) due to mode-collapse. All
models are of the same size to allow for fair comparison. Training
directly on real data (understandably) outperforms training on
generated data by a large margin. Surprisingly however,
training on ARAE samples outperforms training on LM/AE samples in terms of reverse perplexity.
\begin{table}[t]
  \tiny
  \centering
  \begin{tabular}{lp{0.78\linewidth}}    \toprule
Positive  & great indoor mall .  \\
$ \Rightarrow $ ARAE & no smoking mall .  \\
$ \Rightarrow $ Cross-AE & terrible outdoor urine . \\\\

Positive  & it has a great atmosphere , with wonderful service . \\
$ \Rightarrow $ ARAE & it has no taste , with a complete jerk .  \\
$ \Rightarrow $ Cross-AE & it has a great horrible food and run out service . \\\\

Positive  & we came on the recommendation of a bell boy and the food was amazing . \\
$ \Rightarrow $ ARAE & we came on the recommendation and the food was a joke .  \\
$ \Rightarrow $ Cross-AE & we went on the car of the time and the chicken was awful . \\
    \midrule
Negative  &   hell no ! \\
$ \Rightarrow $ ARAE & hell great ! \\
$ \Rightarrow $ Cross-AE & incredible pork ! \\\\

Negative  & small , smokey , dark and rude management . \\
$ \Rightarrow $ ARAE &  small , intimate , and cozy friendly staff . \\
$ \Rightarrow $ Cross-AE & great , , , chips and wine . \\\\

Negative  &  the people who ordered off the menu did n't seem to do much better . \\
$ \Rightarrow $ ARAE & the people who work there are super friendly and the menu is good . \\
$ \Rightarrow $ Cross-AE & the place , one of the office is always worth you do a business . \\
    \bottomrule
  \end{tabular}
 \vspace{-2mm}
  \caption{\label{tab:sent_trans} Sentiment transfer results, where we transfer from positive to negative sentiment (Top) and negative to positive sentiment (Bottom). Original sentence and transferred output (from ARAE and the Cross-Aligned AE (from \citet{Shen2017}) of 6 randomly-drawn examples. }
      \vspace{-8mm}
\end{table}

\begin{table}[t]
  \small
  \centering
  \begin{tabular}{lcccc}
    \toprule
    & \multicolumn{4}{c}{Automatic Evaluation} \\
    Model  &  Transfer & BLEU & Forward &  Reverse \\
    \midrule
     Cross-Aligned AE & 77.1\%  & 17.75 & 65.9 & 124.2  \\
     AE   & 59.3\% & 37.28 & 31.9 & 68.9 \\
     ARAE, $\lambda^{(1)}_a$ & 73.4\% & 31.15 & 29.7 & 70.1 \\
     ARAE, $\lambda^{(1)}_b$ & 81.8\% & 20.18 & 27.7 & 77.0 \\
    \bottomrule
  \end{tabular}
  \vspace{0.15cm}

  \begin{tabular}{lccc}
    \toprule
    & \multicolumn{3}{c}{Human Evaluation} \\

    Model &  Transfer & Similarity  & Naturalness\\
    \midrule
     Cross-Aligned AE & 57\%  & 3.8 & 2.7 \\
     ARAE, $\lambda^{(1)}_b$ & 74\% & 3.7 & 3.8 \\
    \bottomrule
  \end{tabular}
    \vspace{-1mm}
  \caption{\label{tab:sentiment} Sentiment transfer. (Top) Automatic metrics (Transfer/BLEU/Forward PPL/Reverse PPL), (Bottom) Human evaluation metrics (Transfer/Similarity/Naturalness). Cross-Aligned AE is from \citet{Shen2017}}
      \vspace{-6mm}
\end{table}

\begin{table}[h]
  \tiny
  \centering
  \begin{tabular}{lp{0.75\linewidth}}
    \toprule
Science   & what is an event horizon with regards to black holes ?  \\
$\Rightarrow$ Music & what is your favorite sitcom with adam sandler ? \\
$\Rightarrow$ Politics & what is an event with black people ? \\
\\
Science  & take 1ml of hcl ( concentrated ) and dilute it to 50ml . \\
$\Rightarrow$ Music & take em to you and shout it to me  \\
$\Rightarrow$ Politics & take bribes to islam and it will be punished .  \\
\\
Science  & just multiply the numerator of one fraction by that of the other . \\
$\Rightarrow$ Music & just multiply the fraction of the other one that \&apos;s just like it .  \\
$\Rightarrow$ Politics & just multiply the same fraction of other countries . \\
    \midrule
Music  & do you know a website that you can find people who want to join bands ? \\
$\Rightarrow$ Science & do you know a website that can help me with science ? \\
$\Rightarrow$ Politics & do you think that you can find a person who is in prison ? \\
\\
Music  & all three are fabulous artists , with just incredible talent ! ! \\
$\Rightarrow$ Science & all three are genetically bonded with water , but just as many substances , are capable of producing a special case . \\
$\Rightarrow$ Politics & all three are competing with the government , just as far as i can . \\
\\
Music  & but there are so many more i can \&apos;t think of !  \\
$\Rightarrow$ Science & but there are so many more of the number of questions . \\
$\Rightarrow$ Politics & but there are so many more of the can i think of today . \\
    \midrule
 Politics   & republicans : would you vote for a cheney / satan ticket in 2008 ? \\
$\Rightarrow$ Science & guys : how would you solve this question ? \\
$\Rightarrow$ Music & guys : would you rather be a good movie ? \\
\\
Politics   & 4 years of an idiot in office + electing the idiot again = ? \\
$\Rightarrow$ Science & 4 years of an idiot in the office of science ? \\
$\Rightarrow$ Music & 4 ) <unk> in an idiot , the idiot is the best of the two points ever ! \\
\\
Politics   & anyone who doesnt have a billion dollars for all the publicity cant win . \\
$\Rightarrow$ Science & anyone who doesnt have a decent chance is the same for all the other . \\
$\Rightarrow$ Music & anyone who doesnt have a lot of the show for the publicity . \\
    \bottomrule
  \end{tabular}

  \caption{\label{tab:yahoo_trans} Topic Transfer. Random samples from the Yahoo dataset. Note the first row is from ARAE trained on titles while the following ones are from replies.}
      \vspace{-4mm}
\end{table}

\vspace{-2mm}
\subsection{Unaligned Text Style Transfer}
\vspace{-1mm}
Next we evaluate the model in the context of a learned adversarial
prior, as described in Section~\ref{sub:codetrans}. We
experiment with two unaligned text transfer tasks: (i) transfer of sentiment on the Yelp corpus, and (ii) topic on the
Yahoo corpus \citep{zhang2015character}.  For sentiment we follow the
setup of \citet{Shen2017} and split the Yelp corpus into two sets of
unaligned positive and negative reviews. We train ARAE  with two separate decoder RNNs, one for positive,
$p(\boldx\ |\ \boldz, y=1)$, and one for negative sentiment
$p(\boldx\ |\ \boldz, y=0)$, and incorporate adversarial training of
the encoder to remove sentiment information from the prior. Transfer
corresponds to encoding sentences of one class and decoding,
greedily, with the opposite decoder.
Experiments compare against the cross-aligned AE of \citet{Shen2017}
and also an AE trained without the adversarial regularization. For
ARAE, we experimented with different $\lambda^{(1)}$ weighting on the
adversarial loss (see section 4) with
$\lambda_a^{(1)} = 1, \lambda_b^{(1)} = 10$. Both use
$\lambda^{(2)} = 1$. Empirically the adversarial regularization
enhances transfer and perplexity, but tends to make the transferred
text less similar to the original, compared to the AE.  Randomly
selected example sentences are shown in Table~\ref{tab:sent_trans}
and additional outputs are available in Appendix~\ref{app:sheet}.

Table \ref{tab:sentiment} (top) shows quantitative evaluation. We use four
automatic metrics: (i) Transfer: how successful the model
is at altering sentiment based on an automatic classifier (we use the
\texttt{fastText} library \citep{joulin2016bag});  (ii) BLEU: the
consistency between the transferred text and the original; (iii)
Forward PPL: the fluency of the generated text; (iv) Reverse PPL: measuring the extent to which the generations are
representative of the underlying data distribution. Both
perplexity numbers are obtained by training an RNN language model.
Table~\ref{tab:sentiment}~(bottom) shows human evaluations on the
cross-aligned AE and our best ARAE model. We randomly select 1000
sentences (500/500 positive/negative), obtain the corresponding
transfers from both models, and ask crowdworkers to evaluate the
sentiment (Positive/Neutral/Negative) and naturalness (1-5, 5 being
most natural) of the transferred sentences. We create a separate task
in which we show the original and the transferred sentences, and ask
them to evaluate the similarity based on sentence structure (1-5, 5
being most similar). We explicitly requested that the reader disregard
sentiment in similarity assessment.

The same method can be applied to other style transfer tasks, for
instance the more challenging Yahoo QA data
\citep{zhang2015character}.  For Yahoo we chose 3 relatively distinct
topic classes for transfer: \textsc{Science \& Math}, \textsc{Entertainment \& Music}, and
\textsc{Politics \& Government}.  As the dataset contains both questions and
answers, we separated our experiments into titles (questions) and
replies (answers). Randomly-selected generations are shown in Table
\ref{tab:yahoo_trans}. See Appendix \ref{app:sheet} for additional generation examples.

\begin{table}
\small
  \centering
 \begin{tabular}{lccc}
    \toprule
    Model  &  Medium & Small & Tiny \\
    \midrule
    Supervised Encoder  & 65.9\%  & 62.5\% & 57.9\% \\
    Semi-Supervised AE   & 68.5\% & 64.6\% & 59.9\%\\
    Semi-Supervised ARAE & 70.9\% & 66.8\% & 62.5\% \\
    \bottomrule
  \end{tabular}
  \caption{  \label{tab:semi} Semi-Supervised accuracy on the natural language inference
    (SNLI) test set, respectively using 22.2\% (medium), 10.8\% (small), 5.25\% (tiny) of the supervised labels of the full SNLI training set (rest used for unlabeled AE training). }
    \vspace{-6mm}
\end{table}

\begin{figure*}[h]
  \centering
\minipage{0.3\linewidth}
\includegraphics[scale=0.22]{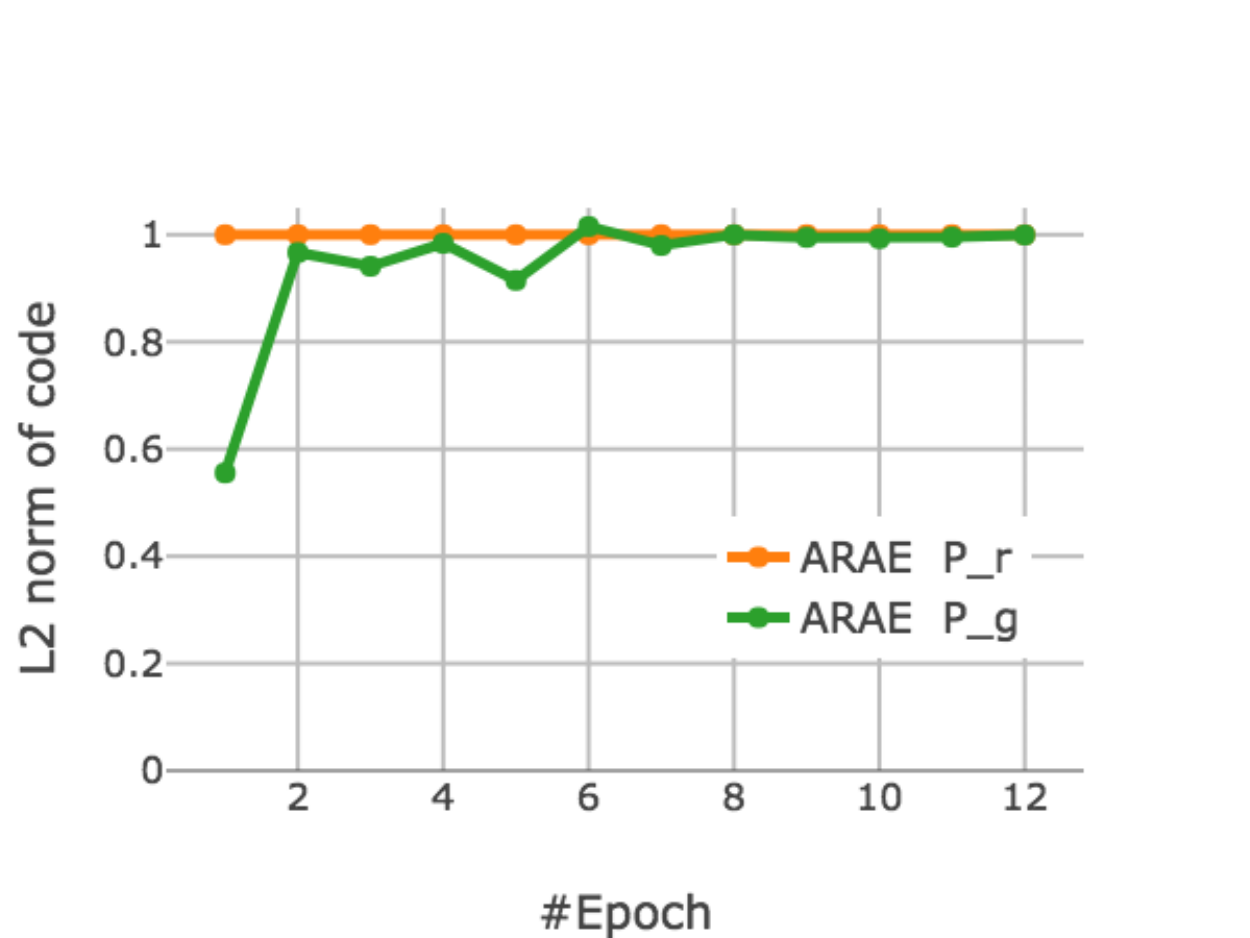}
\endminipage \hspace{1pt}
\minipage{0.3\linewidth}
\includegraphics[scale=0.22]{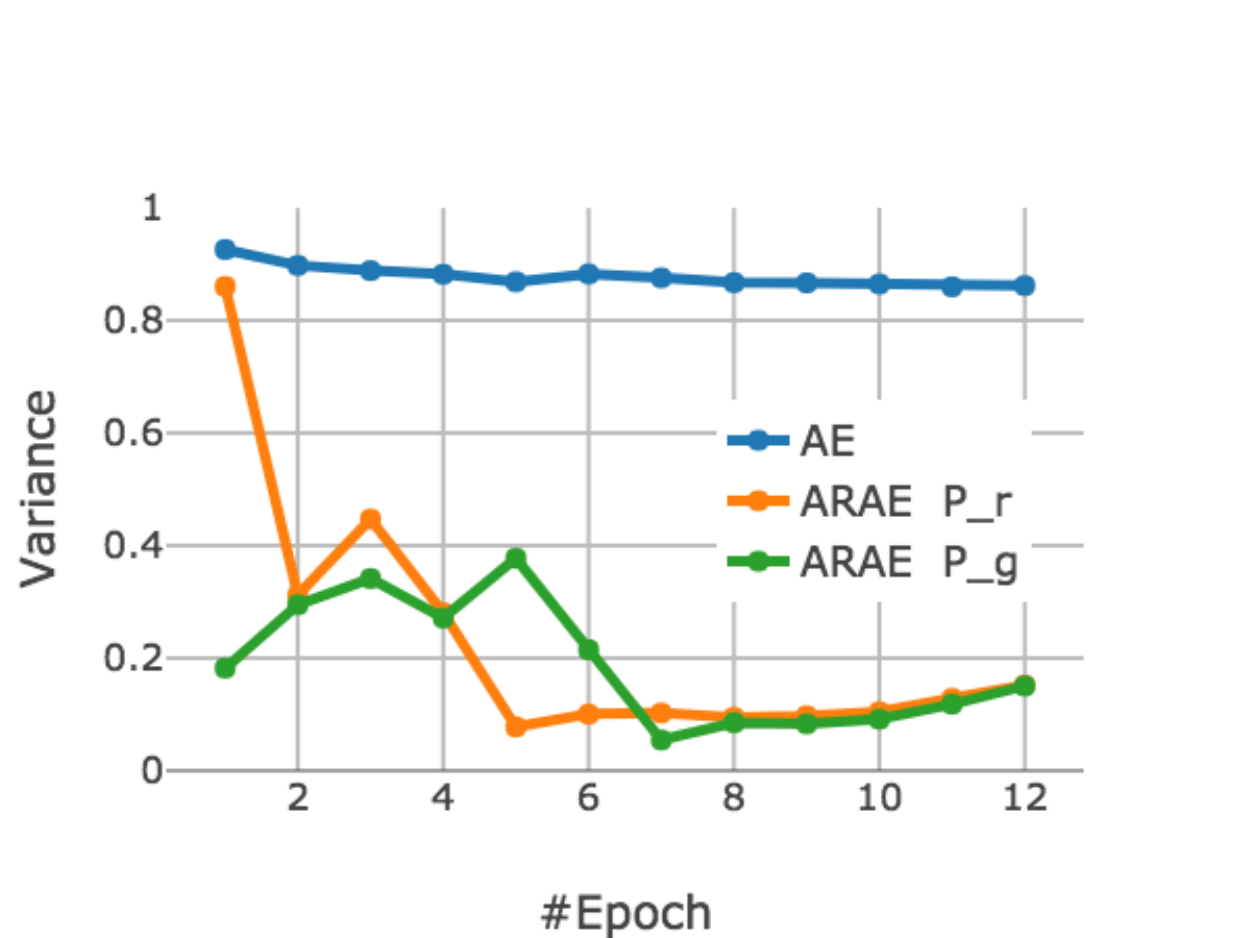}
\endminipage \hspace{1pt}
\minipage{0.3\linewidth}
\includegraphics[scale=0.22]{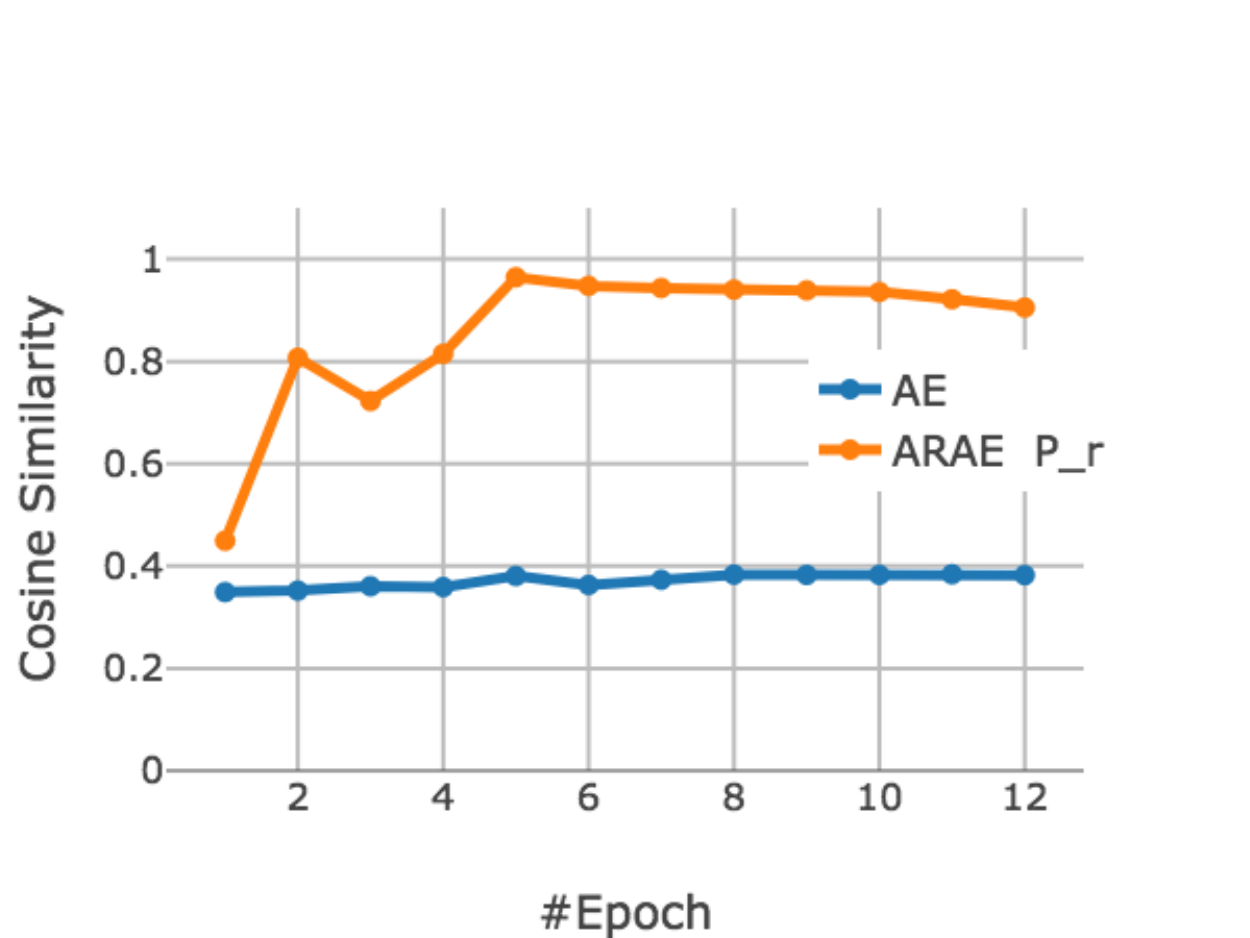}
\endminipage
  \caption{\small
    Left: $\ell_2$ norm of encoder output $\boldz$ and generator output $\tboldz$ during ARAE training. ($\boldz$
    is normalized, whereas the generator learns to match).
	Middle: Sum of the dimension-wise variances of $\boldz$ and generator codes $\tboldz$ as well as reference AE.
    Right: Average cosine similarity of nearby sentences (by word edit-distance) for the ARAE and AE during training.
    \label{fig:variance}}
        \vspace{-3mm}
\end{figure*}

\vspace{-3mm}
\subsection{Semi-Supervised Training}
\vspace{-1mm}
Latent variable models can also provide an easy method for
semi-supervised training. We use a natural language inference task to
compare semi-supervised ARAE with other training methods. Results are shown in
Table~\ref{tab:semi}.  The full SNLI training set contains 543k sentence pairs, and we use supervised
sets of 120k (Medium), 59k (Small), and 28k (Tiny) and use the rest of the training set
for unlabeled training.  As a baseline we use an AE trained on the additional data,
similar to the setting explored in \citet{dai2015semi}. For ARAE we use
the subset of unsupervised data of length $<15$ (i.e. ARAE is trained on less data than AE for
unsupervised training).
The results are shown in Table~\ref{tab:semi}.
Training on unlabeled data with an AE objective improves upon a model just trained on labeled data. Training with adversarial regularization provides further gains.

\vspace{-2mm}
\section{Discussion}
\paragraph{Impact of Regularization on Discrete Encoding}
\label{subsec:smoothness}

We further examine the impact of adversarial regularization on the
encoded representation produced by the model as it is trained.
Figure~\ref{fig:variance} (left), shows a sanity check that the
$\ell_2$ norm of encoder output $\boldz$ and prior samples $\tboldz$
converge quickly in ARAE training. The middle plot compares the trace
of the covariance matrix between these terms as training
progresses. It shows that variance of the encoder and the prior match
after several epochs.

\begin{figure}
  \small
   \begin{tabular}{lcc}
    \toprule
     $k$ & AE & ARAE  \\
     \midrule
     0 & 1.06 & 2.19 \\
     1 & 4.51 & 4.07 \\
     2 & 6.61 & 5.39 \\
     3 & 9.14 & 6.86 \\
     4 & 9.97 & 7.47 \\
     \bottomrule
  \end{tabular}\hspace*{0.2cm}
  \begin{tiny}
  \begin{tabular}{ll}
     \toprule
    Model & Samples \\
    \midrule
     Original & A woman wearing sunglasses  \\
     Noised & A woman sunglasses wearing   \\
    AE & A woman sunglasses wearing sunglasses \\
     ARAE & A woman wearing sunglasses  \\
     \midrule
     Original & Pets galloping down the street  \\
     Noised & Pets down the galloping street  \\
     AE & Pets riding the down galloping   \\
     ARAE & Pets congregate down the street near a ravine  \\
     \bottomrule
   \end{tabular}
  \end{tiny}

  \caption{\label{fig:manisent}
   Reconstruction error (negative log-likelihood averaged
   over sentences) of the original sentence from a corrupted
   sentence. Here $k$ is the number
   of swaps performed on the original sentence.}
       \vspace{-5mm}
\end{figure}
\vspace{-3mm}
\paragraph{Smoothness and Reconstruction}

We can also assess the ``smoothness'' of the encoder model learned
ARAE \citep{Rifai2011}. We start with a simple proxy that a smooth encoder
model should map similar sentences to similar $\mathbf{z}$ values. For
250 sentences, we calculate the average cosine similarity of 100
randomly-selected sentences within an edit-distance of at most 5 to
the original. The graph in Figure~\ref{fig:variance} (right) shows
that the cosine similarity of nearby sentences is quite high for ARAE
compared to a standard AE and increases in early rounds of training.
To further test this property, we feed noised discrete input to the encoder and
(i) calculate the score given to the original input, and (ii) compare
the resulting reconstructions.  Figure~\ref{fig:manisent} (right) shows
results for text where $k$ words are first permuted in each
sentence. We observe that ARAE is able to map a noised sentence to
a natural sentence (though not necessarily the denoised
sentence). Figure~\ref{fig:manisent} (left) shows empirical results for
these experiments. We obtain the reconstruction error (negative
log likelihood) of the original non-noised sentence under the
decoder, utilizing the noised code. We find that when $k = 0$ (i.e.
no swaps), the regular AE better reconstructs the exact input.
However, as the number of swaps pushes the input further
away, ARAE is more likely to produce the original sentence. (Note
that unlike denoising autoencoders which require a domain-specific
noising function \citep{Hill2016,Vincent2008}, the ARAE is not
explicitly trained to denoise an input.)

\vspace{-4mm}
\paragraph{Manipulation through the Prior}

An interesting property of latent variable models such as VAEs and
GANs is the ability to manipulate output samples through the prior.
In particular, for ARAE, the Gaussian form of the noise sample $\bolds$
induces the ability to smoothly interpolate between outputs by
exploiting the structure. While language models may provide a better
estimate of the underlying probability space, constructing this style
of interpolation would require combinatorial search, which makes this
a useful feature of latent variable text models. In Appendix~\ref{app:interpolations}
we show interpolations from for the text model, while Figure~\ref{fig:mnist} (bottom)
shows the interpolations for discretized MNIST ARAE.

A related property of GANs is the ability to move in
the latent space via offset vectors.\footnote{Similar to the case with word
vectors \citep{Mikolov2013}, \citet{Radford2016} observe that when the mean latent vector for
``men with glasses'' is subtracted from the mean latent vector for
``men without glasses'' and applied to an image of a ``woman without
glasses'', the resulting image is that of a ``woman with glasses''.}
To experiment with this property we generate sentences from
the ARAE and compute vector transforms in this space to attempt
to change main verbs, subjects and modifier (details in Appendix \ref{app:arithmetic}).
Some examples of successful transformations are shown in Figure~\ref{tab:shift} (bottom).
Quantitative evaluation of the success of the vector transformations
is given in Figure~\ref{tab:shift} (top).
\begin{figure}[t]
  \centering
  \tiny
    \begin{tabular}{lcc}
    \toprule
    Transform & Match \% & Prec\\
    \midrule
    walking& 85 & 79.5 \\
    man & 92 & 80.2\\
    two & 86 & 74.1\\
    dog & 88& 77.0\\
    standing& 89& 79.3\\
    several& 70& 67.0 \\
    \bottomrule
  \end{tabular}\\
\vspace{3mm}
  \begin{tabular}{ll}
  \toprule
    A man in a tie is sleeping and clapping on balloons .& $\Rightarrow_{\text{walking}}$ \\
    A man in a tie is clapping and \textcolor{red}{walking} dogs .\\ \midrule
    The jewish boy is trying to stay out of his skateboard . & $\Rightarrow_{\text{man}}$ \\
    The jewish \textcolor{red}{man} is trying to stay out of his horse . \\ \midrule
    Some child head a playing plastic with drink . & $\Rightarrow_{\text{Two}}$ \\
    \textcolor{red}{Two} children playing a head with plastic drink .\\ \midrule
    The people shine or looks into an area . & $\Rightarrow_{\text{dog}}$ \\
    The \textcolor{red}{dog} arrives or looks into an area .\\ \midrule
    A women are walking outside near a man . & $\Rightarrow_{\text{standing}}$ \\
    Three women are \textcolor{red}{standing} near a man walking .\\ \midrule
    A side child listening to a piece with steps playing on a table . & $\Rightarrow_{\text{Several}}$ \\
    \textcolor{red}{Several} child playing a guitar on side with a table .\\ 
    \bottomrule
  \end{tabular}
  \caption{Top: Quantitative evaluation of transformations. Match \% refers to the \% of samples where at least one decoder samples (per 100) had the desired
  transformation in the output, while Prec. measures the average precision of the output against the original sentence. Bottom: Examples where the offset vectors produced successful transformations of the original sentence. See Appendix \ref{app:arithmetic} for the full methodology.}
  \label{tab:shift}
\vspace{-4mm}
\end{figure}
\section{Related Work}
\vspace{-1mm}
While ideally autoencoders would learn latent spaces which compactly capture useful features that explain the observed data, in practice they often learn a degenerate \textit{identity} mapping where the latent code space is free of any structure, necessitating the need for some regularization on the latent space.
A popular approach is to regularize through an explicit prior on the code space and use
a variational approximation to the posterior, leading to a family of models called variational autoencoders (VAE) \citep{Kingma2014,Rezende2014}.  Unfortunately VAEs for discrete text sequences can be
challenging to train---for example, if the training procedure is not
carefully tuned with techniques like word dropout and KL annealing
\citep{bowman2015generating}, the decoder simply becomes a language model and
ignores the latent code. However there have been some recent successes through employing
convolutional decoders \cite{Yang2017,Semeniuta2017}, training the latent representation as a topic model \cite{Dieng2017,Wang2018}, using the von Mises--Fisher distribution \cite{Guu2017}, and  combining VAE with iterative inference \cite{Kim2018}. There has also been some work on making
the prior more flexible through explicit parameterization
\citep{Chen2017,Tomczak2017}. A notable technique is
adversarial autoencoders (AAE) \citep{Makhzani2015} which attempt to
imbue the model with a more flexible prior implicitly through
adversarial training. Recent work on Wasserstein autoencoders \cite{tolstikhin2017wasserstein} provides a theoretical foundation for the AAE and shows that AAE minimizes the Wasserstein distance between the data/model distributions.

The success of GANs on  images have led many researchers to consider applying GANs to discrete
data such as text. Policy gradient methods are a
natural way to deal with the resulting non-differentiable generator objective
when training directly in discrete space \citep{Glynn1987,Williams1992}. When trained on text data however,
such methods often require pre-training/co-training with a maximum
likelihood (i.e. language modeling) objective \citep{Che2017,Yu2017,Li2017}.
Another direction of work has been
through reparameterizing the categorical distribution with the
Gumbel-Softmax trick \citep{Jang2017,Maddison2017}---while initial
experiments were encouraging on a synthetic task \citep{Kusner2017},
scaling them to work on natural language is a challenging open problem. There have also been recent related
approaches that work directly with the
soft outputs from a generator \citep{Gulrajani2017,Rajeswar2017,Shen2017,Press2017}.
For example, \citet{Shen2017} exploits adversarial loss for
unaligned style transfer between text by
having the discriminator act on the RNN hidden states and using the soft outputs at each step as input to an RNN generator. 
Our approach instead works entirely in fixed-dimensional continuous space and does not require utilizing RNN hidden states directly. It is therefore also different from methods that discriminate in the joint latent/data space, such as ALI \cite{Dumoulin2017} and BiGAN \cite{Donahue2017}. Finally, our work adds to the
 recent line of work on unaligned style transfer for text \cite{Hu2017,Mueller2017,Li2018,Prabhumoye2018,Yang2018}.
\vspace{-3mm}
\section{Conclusion}
\vspace{-1mm}

We present adversarially regularized autoencoders (ARAE) as a simple approach for training
a discrete structure autoencoder jointly with a code-space generative adversarial network.
Utilizing the Wasserstein autoencoder framework \cite{tolstikhin2017wasserstein}, we also interpret ARAE as learning a latent variable model that minimizes an upper bound on the total variation distance between the data/model distributions.
We find that the model learns an improved autoencoder and exhibits a smooth latent space, as demonstrated by semi-supervised experiments, improvements on text style transfer, and manipulations in the latent space.

We note that
(as has been frequently observed when training GANs) the proposed model seemed to be quite sensitive to hyperparameters, and that we only tested our model on simple structures such as binarized digits and short sentences. \citet{Cifka2018} recently
evaluated a suite of sentence generation models and found that models are quite sensitive to their training setup, and that different models do well on different metrics.
Training deep latent variable models that can robustly model complex discrete structures (e.g. documents) remains an important open issue in the field.

\section*{Acknowledgements}
We thank Sam Wiseman, Kyunghyun Cho, Sam Bowman, Joan Bruna, Yacine Jernite, Martín
Arjovsky, Mikael Henaff, and Michael Mathieu for fruitful discussions. We are particularly grateful to Tianxiao Shen
for providing the results for style transfer.
We also thank the NVIDIA Corporation for the donation of a Titan X
Pascal GPU that was used for this research. Yoon Kim was supported by a gift from Amazon AWS Machine Learning Research. 
{
\footnotesize
\bibliography{icml}}
\bibliographystyle{icml2018}

\clearpage
\appendix
\section{Proof of Corollary 1}
\label{sec:proof}
\begin{corollary*}[Discrete case] Suppose $\boldx \in \mathcal{X}$ where $\mathcal{X}$ is the set of all one-hot vectors of length $n$, and let $f_\psi:\mathcal{Z} \rightarrow \Delta^{n-1}$ be a deterministic function that goes from the latent space $\mathcal{Z}$ to the $n-1$ dimensional simplex $\Delta^{n-1}$.
Further let $G_\psi: \mathcal{Z} \rightarrow \mathcal{X}$ be a deterministic function such that $G_\psi(\boldz)= \argmax_{\mathbf{w} \in \mathcal{X}}\mathbf{w}^\top f_\psi(\boldz)$, and as above let $\Prob_\psi(\boldx\ |\ \boldz)$ be the dirac distribution derived from $G_\psi$ such that $p_\psi(\boldx\ |\ \boldz) = \mathds{1}\{\boldx = G_\psi(\boldz) \}$.  Then the following is an upper bound on $\Vert \Prob_\psi - \Prob_\star \Vert_{\text{TV}}$, the total variation distance between $\Prob_\star$ and $\Prob_\psi$.
\[  \inf_{Q(\boldz\ |\ \boldx) : \Prob_Q = \Prob_\boldz} \mathbb{E}_{\Prob_\star}\mathbb{E}_{Q(\boldz\ |\ \boldx)} \Big[-\frac{2}{\log 2} \log \boldx^\top f_\psi(\boldz)\Big]  \]
\end{corollary*}

\begin{proof} Let our cost function be $c(\boldx,\boldy) = \mathds{1}\{\boldx \ne \boldy\}$.
We first note that for all $\boldx,\boldz$
\begin{align*}
\log 2 \mathds{1}\{ \mathbf{x} \ne \argmax_{\mathbf{w} \in \mathcal{X}} \mathbf{w}^\top f_\psi(\boldz)\}  <  -\log \mathbf{x}^\top f_\psi(\mathbf{z})
\end{align*}
This holds since if $\mathds{1}\{ \mathbf{x} \ne \argmax_{\mathbf{w} \in \mathcal{X}}\mathbf{w}^\top f_\psi(\boldz)\} = 1$, we have $\boldx^\top f_\psi(\boldz) < 0.5$, and $-\log \boldx^\top f_\psi(\boldz) > -\log 0.5 = \log 2 $. If on the other hand   $\mathbf{x} = \argmax_{\mathbf{w} \in \mathcal{X}}\mathbf{w}^\top f_\psi(\boldz)$, then the LHS is $0$ and RHS is always postive since $f_\psi(\boldz) \in \Delta^{n-1}$. Then,
\begin{align*}
& \inf_{Q : \Prob_Q = \Prob_\boldz} \mathbb{E}_{\Prob_\star}\mathbb{E}_{Q(\boldz\ |\ \boldx)} [-\frac{2}{\log 2}\log \boldx^\top f_\psi(\boldz)]   \\
> &\inf_{Q : \Prob_Q = \Prob_\boldz} \mathbb{E}_{\Prob_\star}\mathbb{E}_{Q(\boldz\ |\ \boldx)} [2\mathds{1}\{ \boldx \ne \argmax_{\mathbf{w} \in \mathcal{X}} \mathbf{w}^\top f_\psi(\boldz) \}] \\
= & 2 \inf_{Q : \Prob_Q = \Prob_\boldz}\mathbb{E}_{\Prob_\star}\mathbb{E}_{Q(\boldz\ |\ \boldx)} [\mathds{1}\{ \boldx \ne G_\psi(\boldz) \}]  \\
= & 2 \inf_{Q : \Prob_Q = \Prob_\boldz} \mathbb{E}_{\Prob_\star}\mathbb{E}_{Q(\boldz\ |\ \boldx)} [c(\boldx, G_\psi(\boldz))] \\
= & \,2W_c(\Prob_\star, \Prob_\psi)\\
= &  \, \Vert \Prob_\star - \Prob_\psi \Vert_{\text{TV}}
\end{align*}
The fifth line follows from Theorem 1, and the last equality uses the well-known correspondence between total variation distance and optimal transport with the indicator cost function \cite{Gozlan2010}.
\end{proof}

\vspace{-3mm}
\section{Optimality Property}
\vspace{-1mm}
\label{sec:app}
One can interpret the ARAE framework as a dual pathway network mapping two distinct distributions into a similar one; $\text{enc}_{\phi}$ and $g_\theta$ both output code vectors that are kept similar in terms of Wasserstein distance as measured by the critic.
We provide the following proposition showing that under our parameterization of the encoder and the generator, as the Wasserstein distance converges, the encoder distribution ($\Prob_Q$) converges to the generator distribution ($\Prob_z$), and further, their moments converge.

This is ideal since under our setting the generated distribution is simpler than the encoded distribution,
because the input to the generator is from a simple distribution (e.g. spherical Gaussian) and the generator possesses less capacity than the encoder.
However, it is not so simple that it is overly restrictive (e.g. as in VAEs).
Empirically we observe that the first and second moments do indeed converge as
training progresses (Section~\ref{subsec:smoothness}).

\begin{prop}
\label{prop:opt}
Let $\Prob$ be a distribution on a compact set $\rchi$, and $(\Prob_n)_{n \in N}$ be a sequence of distributions on $\rchi$. Further suppose that $ W(\Prob_n, \Prob) \to 0 $. Then the following statements hold:

\begin{enumerate}[label=(\roman*)]
\item $ \Prob_n \squi \Prob$ (i.e. convergence in distribution).
\item All moments converge, i.e. for all $k > 1, k \in \mathbb{N}$,
$$ \E_{X \sim \Prob_n} \Big[\prod_{i=1}^d X_i^{p_i}\Big]  \to \E_{X \sim \Prob}\Big[\prod_{i=1}^d X_i^{p_i}\Big] $$ for all $p_1, \dots, p_d$ such that
$\sum_{i=1}^d p_i = k$
\end{enumerate}
\end{prop}

\begin{proof}
(i) has been proved in \cite{villani2008optimal} Theorem 6.9.

For (ii), using {\it The Portmanteau Theorem}, (i) is equivalent to the following statement:

$\E_{X \sim \Prob_n}[f(X)] \to \E_{X \sim \Prob}[f(X)]$ for all bounded and continuous function $f$: $\reals^{d} \to \reals$, where $d$ is the dimension of the random variable.

 The $k$-th moment of a distribution is given by

$$ \E \Big[\prod_{i=1}^d X_i^{p_i}\Big] \text{ such that } \sum_{i=1}^d p_i = k$$

Our encoded code is bounded as we normalize the encoder output to lie on the unit sphere, and our generated code is also bounded to lie in $(-1,1)^n$ by the $\tanh$ function. Hence
$f(X) = \prod_{i=1}^d X_i^{q_i}$ is a bounded continuous function for all $q_i \ge 0$.
Therefore,

$$ \E_{X \sim \Prob_n} \Big[\prod_{i=1}^d X_i^{p_i}\Big]  \to \E_{X \sim \Prob}\Big[\prod_{i=1}^d X_i^{p_i}\Big] $$

where $\sum_{i=1}^d p_i = k$
\end{proof}

\section{Sample Generations}
In Figure~\ref{fig:text} we show some generated samples from the ARAE, AE, and a LM.
\label{app:generation}
\begin{figure}[h]
  \tiny
\begin{multicols}{3}
\centerline{ \textbf{ARAE Samples}}
A woman preparing three fish .\\
A woman is seeing a man in the river .\\
There passes a woman near birds in the air .\\
Some ten people is sitting through their office .\\
The man got stolen with young dinner bag .\\
Monks are running in court .\\
The Two boys in glasses are all girl .\\
The man is small sitting in two men that tell a children .\\
The two children are eating the balloon animal .\\
A woman is trying on a microscope .\\
The dogs are sleeping in bed .\\

\columnbreak
\centerline{ \textbf{AE Samples}}
Two Three woman in a cart tearing over of a tree . \\
A man is hugging and art . \\
The fancy skier is starting under the drag cup in . \\
A dog are <unk> a \\
A man is not standing . \\
The Boys in their swimming . \\
A surfer and a couple waiting for a show . \\
A couple is a kids at a barbecue . \\
The motorcycles is in the ocean loading \\
I 's bike is on empty \\
The actor was walking in a a small dog area .\\
no dog is young their mother \\

\columnbreak
\centerline{ \textbf{LM Samples}}
a man walking outside on a dirt road , sitting on the dock .\\
A large group of people is taking a photo for Christmas and at night .\\
Someone is avoiding a soccer game .\\
The man and woman are dressed for a movie .\\
Person in an empty stadium pointing at a mountain .\\
Two children and a little boy are <unk> a man in a blue shirt .\\
A boy rides a bicycle .\\
A girl is running another in the forest .\\
the man is an indian women .\\
 \end{multicols}
 \vspace{-4mm}
 \caption{\label{fig:text} Text samples generated from ARAE, a simple AE, and from a baseline LM trained on the same data. To generate from an AE we fit a multivariate Gaussian to the learned code space and generate code vectors from this Gaussian.}
 \vspace{-4mm}
\end{figure}

\section{Sentence Interpolations}
\label{app:interpolations}
In Figure~\ref{fig:interp} we show generations from interpolated latent vectors.  Specifically, we sample two points $\boldz_0$ and $\boldz_1$ from
$p(\boldz)$ and construct intermediary points
$\boldz_{\lambda} = \lambda \boldz_1 + (1-\lambda) \boldz_0 $. For
each we generate the argmax output $\tilde{\boldx}_{\lambda}$.
\begin{figure}
  \tiny

  \begin{multicols}{3}
\textcolor{blue}{A man is on the corner in a sport area . }\\
\textcolor{gray}{A} \textcolor{gray}{man} \textcolor{gray}{is} \textcolor{gray}{on} \textcolor{gray}{corner} \textcolor{gray}{in} \textcolor{gray}{a} road all \textcolor{gray}{.}\\
\textcolor{gray}{A} lady \textcolor{gray}{is} \textcolor{gray}{on} outside \textcolor{gray}{a} racetrack \textcolor{gray}{.}\\
\textcolor{gray}{A} \textcolor{gray}{lady} \textcolor{gray}{is} \textcolor{gray}{outside} \textcolor{gray}{on} \textcolor{gray}{a} \textcolor{gray}{racetrack} \textcolor{gray}{.}\\
\textcolor{gray}{A} lot of people \textcolor{gray}{is} outdoors in an urban setting \textcolor{gray}{.}\\
\textcolor{gray}{A} \textcolor{gray}{lot} \textcolor{gray}{of} \textcolor{gray}{people} \textcolor{gray}{is} \textcolor{gray}{outdoors} \textcolor{gray}{in} \textcolor{gray}{an} \textcolor{gray}{urban} \textcolor{gray}{setting}
\textcolor{gray}{.}\\
\textcolor{red}{A lot of people is outdoors in an urban setting .}

\columnbreak

\textcolor{blue}{A man is on a ship path with the woman . }\\
\textcolor{gray}{A} \textcolor{gray}{man} \textcolor{gray}{is} \textcolor{gray}{on} \textcolor{gray}{a} \textcolor{gray}{ship} \textcolor{gray}{path} \textcolor{gray}{with} \textcolor{gray}{the} \textcolor{gray}{woman} \textcolor{gray}{.}\\
\textcolor{gray}{A} \textcolor{gray}{man} \textcolor{gray}{is} passing \textcolor{gray}{on} \textcolor{gray}{a} bridge \textcolor{gray}{with} \textcolor{gray}{the} girl \textcolor{gray}{.}\\
\textcolor{gray}{A} \textcolor{gray}{man} \textcolor{gray}{is} \textcolor{gray}{passing} \textcolor{gray}{on} \textcolor{gray}{a} \textcolor{gray}{bridge} \textcolor{gray}{with} \textcolor{gray}{the} \textcolor{gray}{girl} \textcolor{gray}{.}\\
\textcolor{gray}{A} \textcolor{gray}{man} \textcolor{gray}{is} \textcolor{gray}{passing} \textcolor{gray}{on} \textcolor{gray}{a} \textcolor{gray}{bridge} \textcolor{gray}{with} \textcolor{gray}{the} \textcolor{gray}{girl} \textcolor{gray}{.}\\
\textcolor{gray}{A} \textcolor{gray}{man} \textcolor{gray}{is} \textcolor{gray}{passing} \textcolor{gray}{on} \textcolor{gray}{a} \textcolor{gray}{bridge} \textcolor{gray}{with} \textcolor{gray}{the} dogs \textcolor{gray}{.}\\
\textcolor{red}{A man is passing on a bridge with the dogs .}

\columnbreak

\textcolor{blue}{A man in a cave is used an escalator .}\\\\
\textcolor{gray}{A} \textcolor{gray}{man} \textcolor{gray}{in} \textcolor{gray}{a} \textcolor{gray}{cave} \textcolor{gray}{is} \textcolor{gray}{used} \textcolor{gray}{an} \textcolor{gray}{escalator} \\
\textcolor{gray}{A} \textcolor{gray}{man} \textcolor{gray}{in} \textcolor{gray}{a} \textcolor{gray}{cave} \textcolor{gray}{is} \textcolor{gray}{used} chairs \textcolor{gray}{.} \\
\textcolor{gray}{A} \textcolor{gray}{man} \textcolor{gray}{in} \textcolor{gray}{a} number \textcolor{gray}{is} \textcolor{gray}{used} many equipment  \\
\textcolor{gray}{A} \textcolor{gray}{man} \textcolor{gray}{in} \textcolor{gray}{a} \textcolor{gray}{number} \textcolor{gray}{is} posing so on \textcolor{gray}{a} big rock \textcolor{gray}{.} \\
People are \textcolor{gray}{posing} \textcolor{gray}{in} \textcolor{gray}{a} rural area \textcolor{gray}{.}\\
\textcolor{red}{People are posing in a rural area}.

\end{multicols}
\vspace{-4mm}
  \caption{  \label{fig:interp} Sample interpolations from the ARAE. Constructed by linearly
  interpolating in the latent space and decoding to the output space.
  Word changes are highlighted in black.
}
\vspace{-4mm}
\end{figure}

\section{Vector Arithmetic}
\label{app:arithmetic}
We generate 1 million sentences from the ARAE and parse the sentences to obtain the main verb, subject,
and modifier. Then for a given sentence, to change the main verb we
subtract the mean latent vector $(\mathbf{t})$ for all other sentences
with the same main verb (in the first example in
Figure~\ref{tab:shift} this would correspond to all sentences that had
``sleeping'' as the main verb) and add the mean latent vector for all
sentences that have the desired transformation (with the running
example this would be all sentences whose main verb was
``walking''). We do the same to transform the subject and the
modifier.  We decode back into sentence space with the transformed
latent vector via sampling from
$p_{\psi}(g(\mathbf{z} + \mathbf{t}))$. Some examples of successful
transformations are shown in Figure~\ref{tab:shift} (right).
Quantitative evaluation of the success of the vector transformations
is given in Figure~\ref{tab:shift} (left). For each original vector
$\mathbf{z}$ we sample 100 sentences from $p_{\psi}(g(\mathbf{z} + \mathbf{t}))$  over the
transformed new latent vector and consider it a match if \textit{any}
of the sentences demonstrate the desired transformation. Match \% is
proportion of original vectors that yield a match post
transformation. As we ideally want the generated samples to only
differ in the specified transformation, we also calculate the average
word precision against the original sentence (Prec) for any match.

\section{Experimental Details}
\label{app:exp}

\subsection*{MNIST experiments}
\label{subsec:hyperparameter}

\begin{itemize}
\item The encoder is a three-layer MLP, \texttt{784-800-400-100}. 
\item Additive Gaussian noise is injected into $\boldc$ then gradually decayed to $0$.
\item The decoder is a four-layer MLP, \texttt{100-400-800-1000-784}
\item The autoencoder is optimized by Adam, with learning rate \texttt{5e-04}.
\item An MLP generator \texttt{32-64-100-150-100}.
\item An MLP critic \texttt{100-100-60-20-1} with weight clipping $\epsilon = 0.05$. The critic is trained 10 iterations in every loop.
\item GAN is optimized by Adam, with learning rate \texttt{5e-04} on the generator and \texttt{5e-05} on the critic.
\item Weighing factor $\lambda^{(1)}=0.2$.
\end{itemize}

\subsection*{Text experiments}
\label{subsec:archtext}
\begin{itemize}
\item The encoder is an one-layer LSTM with 300 hidden units.
\item Additive Gaussian noise is injected into $\boldc$ then gradually decayed to $0$.
\item The decoder is an one-layer LSTM with 300 hidden units.
\item The LSTM state vector is augmented by the hidden code $c$ at every decoding time step, before forwarding into the output softmax layer.
\item The word embedding is of size 300.
\item The autoencoder is optimized by SGD with learning rate 1. A grad clipping on the autoencoder, with max \texttt{grad\_norm} set to 1.
\item An MLP generator \texttt{100-300-300}.
\item An MLP critic \texttt{300-300-1} with weight clipping $\epsilon = 0.01$. The critic is trained 5 iterations in every loop.
\item GAN is optimized by Adam, with learning rate \texttt{5e-05} on the generator, and \texttt{1e-05} on the critic.
\end{itemize}

\subsection*{Semi-supervised experiments}
\label{app:archsemi}
The following changes are made based on the SNLI experiments:
\begin{itemize}
\item Larger network to GAN components: an MLP generator \texttt{100-150-300-500} and an MLP critic \texttt{500-500-150-80-20-1} with weight clipping factor $\epsilon = 0.02$. 
\end{itemize}

\subsection*{Yelp/Yahoo transfer}
\begin{itemize}
\item An MLP style adversarial classifier \texttt{300-200-100}, trained by SGD learning rate $0.1$.
\item Weighing factor from both adversarial forces $\lambda_a^{(1)}=1$, $\lambda_b^{(1)}=10$.
\end{itemize}
\section{Style Transfer Samples}
In the following pages we show randomly sampled style transfers from the Yelp/Yahoo corpus.
\label{app:sheet}
\begin{figure*}[t]
  { \bf  Yelp Sentiment Transfer} \\
  \vspace{5mm}
  \tiny
  \centering
    \begin{tabular}{ll|ll}
    \toprule
    & Positive to Negative & & Negative to Positive \\
    \midrule
Original & great indoor mall . & Original & hell no ! \\
ARAE & no smoking mall . & ARAE & hell great ! \\
Cross-AE & terrible outdoor urine . & Cross-AE & incredible pork ! \\
\midrule
Original & great blooming onion . & Original & highly disappointed ! \\
ARAE & no receipt onion . & ARAE & highly recommended ! \\
Cross-AE & terrible of pie . & Cross-AE & highly clean ! \\
\midrule
Original & i really enjoyed getting my nails done by peter . & Original & bad products . \\
ARAE & i really needed getting my nails done by now . & ARAE & good products . \\
Cross-AE & i really really told my nails done with these things . & Cross-AE & good prices . \\
\midrule
Original & definitely a great choice for sushi in las vegas ! & Original & i was so very disappointed today at lunch . \\
ARAE & definitely a \_num\_ star rating for \_num\_ sushi in las vegas . & ARAE & i highly recommend this place today . \\
Cross-AE & not a great choice for breakfast in las vegas vegas ! & Cross-AE & i was so very pleased to this . \\
\midrule
Original & the best piece of meat i have ever had ! & Original & i have n't received any response to anything . \\
ARAE & the worst piece of meat i have ever been to ! & ARAE & i have n't received any problems to please . \\
Cross-AE & the worst part of that i have ever had had ! & Cross-AE & i have always the desert vet . \\
\midrule
Original & really good food , super casual and really friendly . & Original & all the fixes were minor and the bill ? \\
ARAE & really bad food , really generally really low and decent food . & ARAE & all the barbers were entertaining and the bill did n't disappoint . \\
Cross-AE & really good food , super horrible and not the price . & Cross-AE & all the flavors were especially and one ! \\
\midrule
Original & it has a great atmosphere , with wonderful service . & Original & small , smokey , dark and rude management . \\
ARAE & it has no taste , with a complete jerk . & ARAE & small , intimate , and cozy friendly staff . \\
Cross-AE & it has a great horrible food and run out service . & Cross-AE & great , , , chips and wine . \\
\midrule
Original & their menu is extensive , even have italian food . & Original & the restaurant did n't meet our standard though . \\
ARAE & their menu is limited , even if i have an option . & ARAE & the restaurant did n't disappoint our expectations though . \\
Cross-AE & their menu is decent , i have gotten italian food . & Cross-AE & the restaurant is always happy and knowledge . \\
\midrule
Original & everyone who works there is incredibly friendly as well . & Original & you could not see the stage at all ! \\
ARAE & everyone who works there is incredibly rude as well . & ARAE & you could see the difference at the counter ! \\
Cross-AE & everyone who works there is extremely clean and as well . & Cross-AE & you could definitely get the fuss ! \\
\midrule
Original & there are a couple decent places to drink and eat in here as well . & Original & room is void of all personality , no pictures or any sort of decorations . \\
ARAE & there are a couple slices of options and \_num\_ wings in the place . & ARAE & room is eclectic , lots of flavor and all of the best . \\
Cross-AE & there are a few night places to eat the car here are a crowd . & Cross-AE & it 's a nice that amazing , that one 's some of flavor . \\
\midrule
Original & if you 're in the mood to be adventurous , this is your place ! & Original & waited in line to see how long a wait would be for three people . \\
ARAE & if you 're in the mood to be disappointed , this is not the place . & ARAE & waited in line for a long wait and totally worth it . \\
Cross-AE & if you 're in the drive to the work , this is my place ! & Cross-AE & another great job to see and a lot going to be from dinner . \\
\midrule
Original & we came on the recommendation of a bell boy and the food was amazing . & Original & the people who ordered off the menu did n't seem to do much better . \\
Cross-AE & we came on the recommendation and the food was a joke . & ARAE & the people who work there are super friendly and the menu is good . \\
Cross-AE & we went on the car of the time and the chicken was awful . & Cross-AE & the place , one of the office is always worth you do a business . \\
\midrule
Original & service is good but not quick , just enjoy the wine and your company . & Original & they told us in the beginning to make sure they do n't eat anything . \\
ARAE & service is good but not quick , but the service is horrible . & ARAE & they told us in the mood to make sure they do great food . \\
Cross-AE & service is good , and horrible , is the same and worst time ever . & Cross-AE & they 're us in the next for us as you do n't eat . \\
\midrule
Original & the steak was really juicy with my side of salsa to balance the flavor . & Original & the person who was teaching me how to control my horse was pretty rude . \\
ARAE & the steak was really bland with the sauce and mashed potatoes . & ARAE & the person who was able to give me a pretty good price . \\
Cross-AE & the fish was so much , the most of sauce had got the flavor . & Cross-AE & the owner 's was gorgeous when i had a table and was friendly . \\
\midrule
Original & other than that one hell hole of a star bucks they 're all great ! & Original & he was cleaning the table next to us with gloves on and a rag . \\
ARAE & other than that one star rating the toilet they 're not allowed . & ARAE & he was prompt and patient with us and the staff is awesome . \\
Cross-AE & a wonder our one came in a \_num\_ months , you 're so better ! & Cross-AE & he was like the only thing to get some with with my hair . \\
    \bottomrule
  \end{tabular}
  \caption{ Full sheet of sentiment transfer result on the Yelp corpus.}
\end{figure*}

\newpage

\begin{figure*}[h]
{ \bf  Yahoo Topic Transfer on Questions}  \\
  \vspace{5mm}
  \tiny
  \centering
  \begin{tabular}{lp{0.25\textwidth}|lp{0.25\textwidth}|lp{0.25\textwidth}}
    \toprule
    & from Science & & from Music & & from Politics \\
    \midrule
Original & what is an event horizon with regards to black holes ? & Original & do you know a website that you can find people who want to join bands ? & Original & republicans : would you vote for a cheney / satan ticket in 2008 ? \\
Music & what is your favorite sitcom with adam sandler ? & Science & do you know a website that can help me with science ? & Science & guys : how would you solve this question ? \\
Politics & what is an event with black people ? & Politics & do you think that you can find a person who is in prison ? & Music & guys : would you rather be a good movie ? \\
\midrule
Original & what did john paul jones do in the american revolution ? & Original & do people who quote entire poems or song lyrics ever actually get chosen best answer ? & Original & if i move to the usa do i lose my pension in canada ? \\
Music & what did john lennon do in the new york family ? & Science & do you think that scientists learn about human anatomy and physiology of life ? & Science & if i move the <unk> in the air i have to do my math homework ? \\
Politics & what did john mccain do in the next election ? & Politics & do people who knows anything about the recent issue of <unk> leadership ? & Music & if i move to the music do you think i feel better ? \\
\midrule
Original & can anybody suggest a good topic for a statistical survey ? & Original & from big brother , what is the girls name who had <unk> in her apt ? & Original & what is your reflection on what will be our organizations in the future ? \\
Music & can anybody suggest a good site for a techno ? & Science & in big bang what is the <unk> of <unk> , what is the difference between <unk> and <unk> ? & Science & what is your opinion on what will be the future in our future ? \\
Politics & can anybody suggest a good topic for a student visa ? & Politics & is big brother in the <unk> what do you think of her ? & Music & what is your favorite music videos on the may i find ? \\
\midrule
Original & can a kidney infection effect a woman \&apos;s <unk> cycle ? & Original & where is the tickets for the filming of the suite life of zack and cody ? & Original & wouldn \&apos;t it be fun if we the people veto or passed bills ? \\
Music & can anyone give me a good film <unk> ? & Science & where is the best place of the blood stream for the production of the cell ? & Science & isnt it possible to be cloned if we put the moon or it ? \\
Politics & can a landlord officer have a <unk> <unk> ? & Politics & where is the best place of the navy and the senate of the union ? & Music & isnt it possible or if we \&apos;re getting married ? \\
\midrule
Original & where does the term \&quot; sweating <unk> \&quot; come from ? & Original & the <unk> singers was a band in 1963 who had a hit called <unk> man ? & Original & can anyone tell me how i could go about interviewing north vietnamese soldiers ? \\
Music & where does the term \&quot; <unk> \&quot; come from ? & Science & the <unk> river in a <unk> was created by a <unk> who was born in the last century ? & Science & can anyone tell me how i could find how to build a robot ? \\
Politics & where does the term \&quot; <unk> \&quot; come from ? & Politics & the <unk> are <unk> in a <unk> who was shot an <unk> ? & Music & can anyone tell me how i could find out about my parents ? \\
\midrule
Original & what other <unk> sources are there than burning fossil fuels . & Original & what is the first metal band in the early 60 \&apos;s ..... ? ? ? ? & Original & if the us did not exist would the world be a better place ? \\
Music & what other <unk> are / who are the greatest guitarist currently on tv today ? & Science & what is the first country in the universe ? & Science & if the world did not exist , would it be possible ? \\
Politics & what other <unk> are there for veterans who lives ? & Politics & who is the first president in the usa ? ? ? ? ? ? ? ? ? ? ? ? ? ? ? ? ? ? ? ? ? ? ? ? & Music & if you could not have a thing who would it be ? \\
\midrule
    \bottomrule
  \end{tabular}
  \caption{\label{tab:yahoo_tit_trans} Full sheet of Yahoo topic transfer on titles.}
\end{figure*}

\newpage
\begin{figure*}[t]
{ \bf  Yahoo Topic Transfer on Answers} \\
  \vspace{5mm}
  \tiny
  \centering
  \begin{tabular}{lp{0.25\textwidth}|lp{0.25\textwidth}|lp{0.25\textwidth}}
    \toprule
    & from Science & & from Music & & from Politics \\
    \midrule
Original & take 1ml of hcl ( concentrated ) and dilute it to 50ml . & Original & all three are fabulous artists , with just incredible talent ! ! & Original & 4 years of an idiot in office + electing the idiot again = ? \\
Music & take em to you and shout it to me & Science & all three are genetically bonded with water , but just as many substances , are capable of producing a special case . & Science & 4 years of an idiot in the office of science ? \\
Politics & take bribes to islam and it will be punished . & Politics & all three are competing with the government , just as far as i can . & Music & 4 ) <unk> in an idiot , the idiot is the best of the two points ever ! \\
\midrule
Original & oils do not do this , they do not \&quot; set \&quot; . & Original & she , too , wondered about the underwear outside the clothes . & Original & send me \$ 100 and i \&apos;ll send you a copy - honest . \\
Music & cucumbers do not do this , they do not \&quot; do \&quot; . & Science & she , too , i know , the clothes outside the clothes . & Science & send me an email and i \&apos;ll send you a copy . \\
Politics & corporations do not do this , but they do not . & Politics & she , too , i think that the cops are the only thing about the outside of the u.s. . & Music & send me \$ 100 and i \&apos;ll send you a copy . \\
\midrule
Original & the average high temps in jan and feb are about 48 deg . & Original & i like rammstein and i don \&apos;t speak or understand german . & Original & wills can be <unk> , or typed and signed without needing an attorney . \\
Music & the average high school in seattle and is about 15 minutes . & Science & i like googling and i don \&apos;t understand or speak . & Science & euler can be <unk> , and without any type of operations , or <unk> . \\
Politics & the average high infantry division is in afghanistan and alaska . & Politics & i like mccain and i don \&apos;t care about it . & Music & madonna can be <unk> , and signed without opening or <unk> . \\
\midrule
Original & the light from you lamps would move away from you at light speed & Original & mark is great , but the guest hosts were cool too ! & Original & hungary : 20 january 1945 , ( formerly a member of the axis ) \\
Music & the light from you tube would move away from you & Science & mark is great , but the water will be too busy for the same reason . & Science & nh3 : 20 january , 78 ( a ) \\
Politics & the light from you could go away from your state & Politics & mark twain , but the great lakes , the united states of america is too busy . & Music & 1966 - 20 january 1961 ( a ) 1983 song \\
\midrule
Original & van <unk> , on the other hand , had some serious issues ... & Original & they all offer terrific information about the cast and characters , ... & Original & bulgaria : 8 september 1944 , ( formerly a member of the axis ) \\
Music & van <unk> on the other hand , had some serious issues . & Science & they all offer insight about the characteristics of the earth , and are composed of many stars . & Science & moreover , 8 \^ 3 + ( x + 7 ) ( x \^ 2 ) = ( a \^ 2 ) \\
Politics & van <unk> , on the other hand , had some serious issues . & Politics & they all offer legitimate information about the invasion of iraq and the u.s. , and all aspects of history . & Music & harrison : 8 september 1961 ( a ) ( 1995 ) \\
\midrule
Original & just multiply the numerator of one fraction by that of the other . & Original & but there are so many more i can \&apos;t think of ! & Original & anyone who doesnt have a billion dollars for all the publicity cant win . \\
Music & just multiply the fraction of the other one that \&apos;s just like it . & Science & but there are so many more of the number of questions . & Science & anyone who doesnt have a decent chance is the same for all the other . \\
Politics & just multiply the same fraction of other countries . & Politics & but there are so many more of the can i think of today . & Music & anyone who doesnt have a lot of the show for the publicity . \\
\midrule
Original & civil engineering is still an umbrella field comprised of many related specialties . & Original & i love zach he is sooo sweet in his own way ! & Original & the theory is that cats don \&apos;t take to being tied up but thats <unk> . \\
Music & civil rights is still an art union . & Science & the answer is he \&apos;s definitely in his own way ! & Science & the theory is that cats don \&apos;t grow up to <unk> . \\
Politics & civil law is still an issue . & Politics & i love letting he is sooo smart in his own way ! & Music & the theory is that dumb but don \&apos;t play <unk> to <unk> . \\
\midrule
Original & h2o2 ( hydrogen peroxide ) naturally decomposes to form o2 and water . & Original & remember the industry is very shady so keep your eyes open ! & Original & the fear they are trying to instill in the common man is based on what ? \\
Music & jackie and brad pitt both great albums and they are my fav . & Science & remember the amount of water is so very important . & Science & the fear they are trying to find the common ancestor in the world . \\
Politics & kennedy and blair hate america to invade them . & Politics & remember the amount of time the politicians are open your mind . & Music & the fear they are trying to find out what is wrong in the song . \\
\midrule
Original & the quieter it gets , the more white noise you can here . & Original & but can you fake it , for just one more show ? & Original & think about how much planning and people would have to be involved in what happened . \\
Music & the fray it gets , the more you can hear . & Science & but can you fake it , just for more than one ? & Science & think about how much time would you have to do . \\
Politics & the gop gets it , the more you can here . & Politics & but can you fake it for more than one ? & Music & think about how much money and what would be <unk> about in the world ? \\
\midrule
Original & h2co3 ( carbonic acid ) naturally decomposes to form water and co2 . & Original & i am going to introduce you to the internet movie database . & Original & this restricts the availability of cash to them and other countries too start banning them . \\
Music & phoebe and jack , he \&apos;s gorgeous and she loves to get him ! & Science & i am going to investigate the internet to google . & Science & this reduces the intake of the other molecules to produce them and thus are too large . \\
Politics & nixon ( captured ) he lied and voted for bush to cause his country . & Politics & i am going to skip the internet to get you checked . & Music & this is the cheapest package of them too . \\
\midrule
    \bottomrule
  \end{tabular}\hspace{1cm}
  \caption{\label{tab:yahoo_ans_trans} Full sheet of Yahoo topic transfer on answers.}
\end{figure*}

\end{document}